\documentclass[11pt]{article}
\pdfoutput=1

\usepackage{microtype}
\usepackage{graphicx}
\usepackage{subfigure}
\usepackage{enumerate}
\usepackage{amsthm,amsmath,amsfonts}
\usepackage{booktabs} 
\usepackage{bm}
\usepackage{diagbox}
\usepackage{algorithm,algorithmic}
\usepackage{thm-restate}
\usepackage[algo2e,ruled,vlined]{algorithm2e}

\usepackage{geometry}
\usepackage{amssymb}
\usepackage{tikz}
\usepackage{url}
\usepackage{setspace}
\usepackage[pdftex,bookmarksnumbered,bookmarksopen,
colorlinks,citecolor=blue,linkcolor=blue,urlcolor=blue]{hyperref}
\usepackage{framed}
\usepackage{xcolor}
\usepackage{soul}
\usepackage{longtable}

\usepackage{times}
\usepackage{enumitem}
\usepackage{varwidth}
\usepackage{graphicx}
\usepackage{wrapfig}
\usepackage{enumerate}
\usepackage{caption}

\usepackage{amssymb}
\usepackage{graphicx}
\usepackage{url}
\usepackage{setspace}
\usepackage{framed}
\usepackage{xcolor}
\usepackage{soul}
\usepackage{longtable}

\usepackage{times}
\usepackage{enumitem}
\usepackage{varwidth}
\usepackage{graphicx}
\usepackage{wrapfig}
\usepackage{enumerate}

\usepackage[utf8]{inputenc} 
\usepackage[T1]{fontenc}    
\usepackage{url}            
\usepackage{booktabs}       
\usepackage{amsfonts}       
\usepackage{nicefrac}       
\usepackage{microtype}      

\usepackage{graphicx}
\usepackage{appendix}
\usepackage{mdwlist}
\usepackage{xspace}
\usepackage{color}
\usepackage{mathrsfs}

\usepackage{booktabs}
\usepackage{comment}

\usepackage{multirow}

\newcommand{\sign}{\textup{\textsf{sign}}}

\newcommand{\Appendix}[1]{the full version for}

\newtheorem{theorem}{Theorem}[section]
\newtheorem{lemma}[theorem]{Lemma}
\newtheorem{proposition}[theorem]{Proposition}

\newtheorem{remark}{Remark}
\newtheorem{example}{Example}

\newtheorem{definition}{Definition}

\newcommand{\z}{\mathbf{z}}

\newcommand{\I}{\mathbf{I}}

\newcommand{\R}{\mathbb{R}}

\renewcommand{\comment}[1]{}

\newcommand{\cD}{\mathcal{D}}
\newcommand{\cF}{\mathcal{F}}
\newcommand{\cG}{\mathcal{G}}
\newcommand{\cH}{\mathcal{H}}

\newcommand{\cL}{\mathcal{L}}
\newcommand{\cP}{\mathcal{P}}

\newcommand{\cY}{\mathcal{Y}}
\newcommand{\cX}{\mathcal{X}}

\newcommand{\bbE}{\mathbb{E}}

\definecolor{colorY}{rgb}{0.7 , 0.7 , 0.2}

\DeclareMathOperator*{\argmax}{argmax}

\title{Learning accurate and interpretable  tree-based models}

\author{{Maria-Florina Balcan}$^*$ \hspace{1cm}    {Dravyansh Sharma}\footnote{Carnegie Mellon University. Correspondence: dravyans@cs.cmu.edu}}

\date{}

\begin{document}
\maketitle

\begin{abstract}
Decision trees and their ensembles are  popular  in machine learning as easy-to-understand models. Several techniques have been proposed in the literature for learning  tree-based classifiers, with different techniques working well for data from different domains. In this work, we develop approaches to design  tree-based learning algorithms given repeated access to data from the same domain. We study multiple formulations covering different aspects and popular techniques for learning decision tree based approaches. We propose novel parameterized classes of node splitting criteria in top-down algorithms, which interpolate between popularly used entropy and Gini impurity based criteria, and provide theoretical bounds on the number of samples needed to learn the splitting function appropriate for the data at hand. We also study the sample complexity of tuning prior parameters in Bayesian decision tree learning, and extend our results to decision tree regression. We further consider the problem of tuning hyperparameters in pruning the decision tree for classical pruning algorithms including min-cost complexity pruning. {In addition, our techniques can be used to optimize the explainability versus accuracy trade-off when using decision trees. We extend our results to tuning popular tree-based ensembles, including random forests and gradient-boosted trees. We demonstrate the significance of our approach on real world datasets by learning data-specific decision trees which are simultaneously more accurate and interpretable.}
\end{abstract}

\section{Introduction}

Decision trees are used widely in operations research, management science, data mining, and machine learning. They are easy to use and understand, and explicitly include the decision rules used in making predictions. Each decision rule is a simple comparsion of a real-valued attribute to a threshold or a categorical attribute against a value from some candidate set. Given their remarkable simplicity, decision trees are widely preferred in  applications where it is important to justify algorithmic decisions with intuitive explanations \cite{rudin2018please}. Furthermore, gradient-boosted decision tree ensembles achieve the state-of-the-art accuracy on tabular data, beating even deep neural networks \cite{grinsztajn2022tree,shwartz2022tabular,mcelfresh2023neural}. However, decades of research on decision trees and their ensembles has resulted in a large suite of candidate approaches for building decision tree models \cite{breiman1984classification,mingers1987expert,quinlan2014c4,quinlan1996learning,kearns1996boosting,mansour1997pessimistic,maimon2014data,chen2016xgboost,prokhorenkova2018catboost}. This raises an important question: how should one select the best approach to build a decision tree model for the relevant problem domain?

The problem of learning optimal decision trees is known to be NP-complete \cite{laurent1976constructing}, and it is unlikely that computationally efficient algorithms exist for learning even approximately optimal trees \cite{koch2024fast,koch2024superconstant}. As such, various different ways have been proposed to build decision trees and their ensembles without any optimality guarantees, and extensive empirical studies show that there is no ``best'' approach \cite{mingers1989empirical,mingers1989empiricalb,murphy1993exploring,esposito1997comparative,quinlan1998miniboosting,murthy1998automatic,opitz1999popular,speiser2019comparison}. Current wisdom from the literature dictates that for any problem at hand, one needs a domain expert to try out, compare and tune various methods to build the best decision tree model for any given problem domain. For instance, the popular Python library Scikit-learn \cite{pedregosa2011scikit} implements both Gini impurity and entropy as candidate {\it splitting criteria} (a crucial component in building the decision trees top-down by deciding which node to split into child nodes), and yet theory suggests another promising candidate \cite{kearns1996boosting} {that achieves  smaller error bounds under the Weak Hypothesis Assumption\footnote{an {\it a priori} assumption on the target function. Roughly speaking, it means that the decision tree node functions are already slightly correlated with the target function.}. It is therefore desirable to determine which approach works better for the data coming from a given domain. With sufficient data, can we automate this tedious manual process? 

In this work we approach this crucial question, and propose ways to build more effective decision trees automatically. Our results show provable learning theoretic guarantees and select methods over larger search spaces than what human experts would typically explore. For example, instead of comparing a small finite number of splitting criteria, we examine learnability over continuously infinite parameterized families that yield more effective decision tree learning algorithms.\looseness-1

We consider the problem where the learner has access to multiple related datasets $D_1, \dots, D_N$ coming from the same problem domain (given by a fixed but unknown distribution $\cD$), and the goal is to design a decision tree (or ensemble) learning algorithm that works well over the distribution $\cD$ using as few datasets ($N$, the sample complexity) as possible. This algorithm design problem is typically formulated as the selection of a hyperparameter from an infinite family. Typically finding the best hyperparameters even on a single problem sample is tedious and computationally intensive, so we would like to bound the number of samples over which we should optimize them, while learning  parameters that generalize well over the distribution generating the problem samples.
We take steps towards systematically unifying, automating and formalizing the process of designing decision tree learning algorithms, in a way that is adaptive to the data domain.\looseness-1

\subsection{Our contributions}

We formulate the problem of designing a decision tree learning algorithm as a hyperparameter selection problem over multiple problem instances coming from the same domain. Under this formulation, we study the sample complexity, i.e.\ the number of problem instances needed to learn a provably good algorithm (hyperparameter) under the statistical learning setting (meaning problem instances are drawn from a fixed but unknown distribution) from several different design perspectives important in the construction of decision trees. A key technical challenge is the non-linearity of boundaries of the piecewise structured dual loss function.\looseness-1

\begin{itemize}[leftmargin=*,topsep=0pt,partopsep=1ex,parsep=1ex]\itemsep=-4pt
    \item We introduce a novel family of node splitting criterion called $(\alpha,\beta)$-Tsallis entropy criterion, which contains two tunable parameters, and includes several popular node splitting criteria from the literature including the entropy-based ID3/C4.5 \cite{quinlan1986induction,quinlan2014c4} and Gini impurity based CART \cite{breiman1984classification}. We  bound the sample complexity of provably tuning these hyperparameters in top-down learning algorithms. We also show how to estimate the parameters $(\alpha,\beta)$ in a computationally efficient manner, by proposing an output-sensitive implementation of the ERM.
    
    \item We further study tuning of parameters in Bayesian decision tree learning algorithms used in generating the prior distribution.  
    We also study a parameterized family for node splitting for regression trees and bound the sample complexity of tuning the parameter.
    \item We next consider the problem of learning the pruning algorithm used in constructing the decision tree. We show how to tune parameters in popular algorithms including the complexity parameter $\tilde{\alpha}$ in the Minimal Cost-Complexity Pruning algorithm, and again obtain sample complexity bounds. We also study the sample complexity of tuning pessimistic error pruning methods, which are computationally faster.
    \item We extend our results to tuning parameters in tree ensembles. For random forests, we bound the sample complexity for selecting the splitting criterion from the $(\alpha,\beta)$-Tsallis family. For gradient-boosted ensembles using XGBoost, we show how to tune the regularization strength.
    \item We consider the problem of optimizing the explainability-accuracy trade-off in the design of decision tree learning algorithms. Here we consider  tuning splitting and pruning parameters simultaneously  when growing a decision tree to size $t$ and pruning it down to size $t'\le t$, while minimizing an objective that incorporates explainability as well as accuracy. Our work is the first to study explainability from a data-driven design perspective.
    \item We further show a how to simultaneously select a parameterized algorithm family and its hyperparameter that corresponds to the best combination, given a collection of tunable algorithms for the same problem.
    \item We perform experiments to show the practical significance and effectiveness of tuning these hyperparameters on real world datasets.
\end{itemize}

\subsection{Related work}

Decision trees \cite{breiman1984classification} predate the development of deep learning based methods, but continue to be an extremely popular tool for data analysis and learning simple explainable models. Recent interest in developing interpretable ``white-box'' models due to concerns around deployment of deep learning in sensitive and critical decision-making have led to a renewed interest in the study of decision trees \cite{rudin2019stop,loyola2019black,molnar2020interpretable}. However, the basic  suite of tools for the design of decision trees has seen little advancement over the decades.  

{\it Building and pruning decision trees}. Typically, decision trees are built in two stages. First the tree is grown in a top-down fashion by successively ``splitting'' existing nodes according to some {\it splitting criterion}. Numerous different methods to select which node to split and how to split have been proposed in the literature \cite{breiman1984classification,quinlan1986induction,quinlan2014c4,kearns1996boosting,larose2014discovering}. The second stage involves pruning the tree to avoid overfitting the training set, and again a variety of approaches are known \cite{breiman1984classification,bohanec1994trading,mingers1987expert,quinlan1987simplifying,mansour1997pessimistic}. Furthermore, empirical work suggests that the appropriate method to use, for both splitting and pruning, depends on the data domain at hand \cite{mingers1989empirical,mingers1989empiricalb}.
The task of selecting the best method or tuning the hyperparameters for a method is left to domain expert. Recent work has developed techniques for computing the optimal decision trees by using branch-and-bound and dynamic programming based techniques \cite{hu2019optimal,lin2020generalized,demirovic2022murtree}. The key idea is to reduce the search space by tracking bounds on the objective value. However, these approaches are computationally more expensive than the classical greedy methods.

{\it Tsallis entropy}. Often in modern applications one needs to solve the classification problem over repeated data instances from the same problem domain. In this work, we take steps to automate the process of algorithm selection for decision tree learning using repeated access to data from the same domain, and also develop more powerful methods for designing decision trees. Our approach is based on Tsallis entropy introduced in the context of statistical physics \cite{tsallis1988possible}, which has been found to be variously useful in machine learning, for example, as a regularizer in reinforcement learning \cite{chow2018path,zimmert2021tsallis}. \cite{khodak2024meta} study tuning of Tsallis entropy in an online meta learning setting for adversarial bandit algorithms. Tsallis entropy based splitting criteria have been empirically studied in the context of decision trees \cite{wang2016improving}. We provide a novel two-parameter generalization that unifies various previously proposed metrics, and provide principled guarantees on the sample complexity of learning the parameters from data. 

{\it Tree ensembles} learn a collection of decision trees and combine their predictions to improve accuracy and reduce overfitting. Two distinct classes of ensembles are popular: bagging-based approaches like random forests, and boosting based methods including gradient-boosted decision trees (GBDTs). {\it Random forests} \cite{breiman2001random} involve training a number of decision trees on different (possibly overlapping) subsets of the training dataset, possibly using different feature subsets. The predictions of the different decision trees in the ensemble are aggregated, for example using a plurality vote in classification, or averaging the predictions in regression. Gradient-boosted trees also build a collection of trees, but by approximately minimizing a regularized loss using gradient-based steps when building the trees. Remarkably, properly tuned gradient-boosted trees outperform neural networks on tabular data, especially for larger datasets \cite{chen2016xgboost,prokhorenkova2018catboost,mcelfresh2023neural}.

{\it Data-driven algorithm design} is a framework  for the design  of algorithms using machine learning in order to optimize performance over  problems coming from a common problem domain \cite{gupta2016pac,balcan2020data}. The approach has been successful  in designing more effective algorithms for a variety of combinatorial problems, ranging from those encountered in machine learning to those in data science (e.g. learning to cluster), numerical linear algebra and game theory \cite{balcan2018learning,balcan2018data,balcan2020learning,balcan2022improved,balcan2023analysis,balcan2024subsidy,balcan2024algorithm,bartlett2022generalization}. The basic premise is to design algorithms for typical inputs instead of worst-case inputs by examining repeated problem instances. In machine learning, this can be used to provably tune hyperparameters \cite{balcan2021data,balcan2022provably,blum2021learning,balcan2025sample} as opposed to employing heuristics like grid search or random search \cite{bergstra2012random} for which formal global-optimality guarantees are typically not known. Moreover, recent work establishes that any data-independent (even non-uniform) discretization performs poorly on certain data distributions, necessitating the need for finding the best hyperparameters over the continuous domain \cite{balcan2024learning}. A key idea is to treat the hyperparameter tuning problem as a statistical learning problem with the parameter space as the hypothesis class and repeated problem samples as data points. Bounding the statistical complexity of this hypothesis class implies sample complexity bounds for hyperparameter tuning using classic learning theory. Another related line of work is the development of output-sensitive computationally efficient algorithms for implementing the Empirical Risk Minimization (ERM) for tuning the hyperparameters \cite{balcan2020learning,balcan2024accelerating}. We limit our attention to the statistical learning setting in this work, an interesting future direction is to extend our results to  online and meta-learning \cite{balcan2018dispersion,sharma2020learning,balcan2021learning,sharma2024no,sharma2025offline}.

{\it Data-driven tree search}. {General techniques have been developed for providing the sample complexity of tuning a linear combination of variable selection policies in branch-and-bound, and special cases of ``path-wise'' node selection policies have been studied.  In contrast, our work provides new technical insights for node selection policies relevant for decision tree learning which do not satisfy the previously studied path-wise properties and involve a more challenging non-linear interpolation. Prior work \cite{balcan2021sample} obtains a general result for tree search without any path-wise assumptions, but  still require a linear interpolation of selection policies.\looseness-1}

\section{Preliminaries and definitions}
\label{sec:prelim}

Let $[k]$ denote the set of integers $\{1,2,\dots,k\}$. A (supervised) classification problem is given by a labeled dataset $D=(X,y)$ over some input domain $X\in\cX^n$ and $y\in\cY^n=[c]^n$ where $c$ denotes the number of distinct classes or categories. Let $\cD$ be a distribution over classification problems of size $n$.\footnote{For simplicity of technical presentation we assume that the dataset size $n$ is fixed across problem instances, but  our sample complexity results hold even without this assumption.} We will consider parameterized families of decision tree learning algorithms, parameterized by some parameter $\rho\in\cP\subseteq \R^d$ and access to datasets $D_1,\dots,D_N\sim \cD^N$. We do not assume that individual data points $(X_i,y_i)$ are i.i.d.\ in any dataset $D_j$. 

We consider a finite {\it node function class} $\cF$ consisting of boolean functions $\cX\rightarrow\{0,1\}$ which are used to label internal nodes in the decision tree, i.e.\ govern given any data point $x\in\cX$ whether the left or right branch should be taken when classifying $x$ using the decision tree. Any given data point $x\in\cX$ corresponds to a unique leaf node determined by the node function evaluations at $x$ along some unique root-to-leaf path. Each leaf node of the decision tree is labeled by a class in $[c]$. Given a dataset $(X,y)$ this leaf label is typically set as the most common label for data points $x\in X$ which are mapped to the leaf node.  

We denote by $T_{l\rightarrow f}$ the tree obtained by {\it splitting} the leaf node $l$, which corresponds to replacing it by an internal node labeled by $f$ and creating two child leaf nodes. We consider a parameterized class of splitting criterion $\cG_{\cP}$ over some parameter space $\cP$ consisting of functions $g_\rho:[0,1]^{c}\rightarrow\R_{\ge0}$ for $\rho\in\cP$. The splitting criterion governs which leaf to be split next and which node function $f\in\cF$ to be used when building the decision tree using a top-down learning algorithm which builds a decision tree by successively splitting nodes using $g_\rho$ until the size equals  input tree size $t$. More precisely, suppose $w(l)$ (the {\it weight} of leaf $l$) denotes the number of data points in $X$ that map to leaf $l$, and 
suppose $p_i(l)$ denotes the fraction of data points labeled by $y=i\in[c]$ among those points that map to leaf $l$. The splitting function over  tree $T$ is given by 

$$G_\rho(T)=\sum_{l\in \mathrm{leaves}(T)}w(l)g_\rho\left(\{p_i(l)\}_{i=1}^c
\right),$$

\noindent and we build the decision tree by successively splitting the leaf nodes using node function $f$ which cause the maximum decrease in the splitting function. For example, the information gain criterion may be expressed using $g_\rho(\{p_i(l)\}_{i=1}^c)=-\sum_{i=1}^cp_i\log p_i$.

\begin{algorithm}[t]
\caption{Top-down decision tree learner ($\cF$, $g_\rho$, t)\label{alg:td-learning}}
\flushleft\textbf{Input}: Dataset $D=(X,y)$\\
\textbf{Parameters}: Node function class $\cF$, splitting criterion $g_\rho\in \cG_{\cP}$, tree size $t$\\
\textbf{Output}: Decision tree $T$\\
\begin{algorithmic}[1] 
\STATE Initialize $T$ to a leaf node labeled by most frequent label $y$ in $D$.
\WHILE{$T$ has at most $t$ internal nodes}
\STATE $l^*,f^*\gets \mathrm{argmin}_{l\in \mathrm{leaves}(T),f\in\cF}G_\rho(T_{l\rightarrow f})$
\STATE $T\gets T_{l^*\rightarrow f^*}$
\ENDWHILE
\STATE \textbf{return} $T$
\end{algorithmic}
\end{algorithm}

Algorithm \ref{alg:td-learning} summarizes this well-known general paradigm. We denote the tree obtained by the top-down decision tree learner on dataset $D$ as $T_{\cF,\rho,t}(D)$.
We study the 0-1 loss of the resulting decision tree classifier. If $T(x)\in[c]$ denotes the prediction of tree $T$ on $x\in \cX$, we define the loss on dataset $D(X,y)$ as

$$L(T,D):=\frac{1}{n}\sum_{i=1}^n\I[T(X_i)\ne y_i],$$

\noindent where $\I[\cdot]$ denotes the 0-1 valued indicator function.

\subsection{Learning theory background and standard results from data-driven algorithm design}
The pseudo-dimension is frequently used to analyze the learning theoretic complexity of real-valued  function classes, and will be a main tool in our sample complexity analysis. For completeness, we include below the formal definition.

\begin{definition}[Shattering and Pseudo-dimension, \cite{anthony1999neural}]\label{def:pdim}
Let $\cF$ be a set of functions mapping from $\cX$ to $\R$, and suppose that $S = \{x_1, \dots, x_m\} \subseteq \cX$. Then $S$ is pseudo-shattered by $\cF$ if there are real numbers $r_1, \dots, r_m$ such that for each $b \in \{0, 1\}^m$ there is a function $f_b$ in $\cF$ with $\mathrm{sign}(f_b(x_i) - r_i) = b_i$ for $i \in [m]$. We say that $r = (r_1, \dots, r_m)$ witnesses the shattering. We say that $\cF$ has pseudo-dimension $d$ if $d$ is the maximum cardinality of a subset $S$ of $\cX$ that is pseudo-shattered by $\cF$, denoted $\mathrm{Pdim}(\cF) = d$. If no such maximum exists, we say that $\cF$ has infinite pseudo-dimension. 
\end{definition}

\noindent Pseudo-dimension is a real-valued analogue of VC-dimension, and is a classic complexity notion in learning theory due to the following theorem which implies the uniform convergence sample complexity for any function in class $\cF$ when $\mathrm{Pdim}(\cF)$ is finite. 
\begin{theorem}[Uniform convergence sample complexity via pseudo-dimension, \cite{anthony1999neural}]\label{thm:pdim}
    Suppose $\cF$ is a class of real-valued functions with range in $[0, H]$ and finite $\mathrm{Pdim}(\cF)$. For every $\epsilon > 0$ and $\delta \in (0, 1)$, given any distribution $\cD$ over $\cX$, with probability $1-\delta$ over the draw of a sample $S\sim\cD^M$, for all functions $f\in\cF$, we have $|\frac{1}{M}\sum_{x\in S}f(x)-\bbE_{x\sim \cD}[f(x)]|\le \epsilon$ for some
    $M=O\left(\left(\frac{H}{\epsilon}\right)^2\left(\mathrm{Pdim}(\cF) + \log\frac{1}{\delta} \right)\right)$. 
\end{theorem}
\noindent 
We also need the following lemma from data-driven algorithm design, which bounds the pseudo-dimension of the class of loss functions, when the dual losses (i.e.\ losses as a function of some algorithmic hyperparameter computed on any fixed problem instance) have a piecewise constant structure with a bounded number of pieces.

\begin{lemma} (Lemma 2.3, \cite{balcan2020data})
     Suppose that for every problem instance $D \in \cD$, the function $L_D(\rho) : \R\rightarrow\R$ is piecewise
constant with at most $N$ pieces. Then the family $\{L_\rho(\cdot)\}$ over instances in $\cD$ has pseudo-dimension $O(\log N)$. \label{lem:ddad-1}
\end{lemma}

\section{Learning to  split nodes in top-down algorithms}\label{sec:split}

In this section, we study the sample complexity of learning the splitting criterion by tuning the parameters in novel as well as previously proposed top-down learning approaches. 

\subsection{$(\alpha,\beta)$-Tsallis entropy}
Given a discrete probability distribution $P=\{p_i\}$ with $\sum_{i=1}^cp_i=1$, we define $(\alpha,\beta)$-Tsallis entropy  as

$$g^{\textsc{Tsallis}}_{\alpha,\beta}(P):=\frac{C}{\alpha-1}\left(1-\left(\sum_{i=1}^cp_i^{\alpha}\right)^{\beta}\right),$$

\noindent where $C$ is a normalizing constant (does not affect Algorithm \ref{alg:td-learning}), $\alpha\in\R^+,\beta\in\mathbb{Z}^+$. $\beta=1$ corresponds to standard Tsallis entropy \cite{tsallis1988possible}.  For example, $\alpha=2,\beta=1$ corresponds to Gini impurity, $\alpha=\frac{1}{2},\beta=2$ corresponds to the Kearns and Mansour criterion (using which error $\epsilon$ can be achieved with trees of size $\mathrm{poly}(1/\epsilon)$, \cite{kearns1996boosting}) and $\lim_{\alpha\rightarrow1}g^{\textsc{Tsallis}}_{\alpha,1}(P)$ yields the (Shannon) entropy criterion. We omit the definitions of these well-known criteria (see Appendix \ref{app:split}, proof of Proposition \ref{prop:interpolation}).  Formally, we show in the following proposition that $(\alpha,\beta)$-Tsallis entropy recovers three popular splitting criteria for appropriate values of $\alpha,\beta$.

\begin{proposition}\label{prop:interpolation}
    The splitting criteria $g^{\textsc{Tsallis}}_{2,1}(P), g^{\textsc{Tsallis}}_{\frac{1}{2},2}(P)$ and $\lim_{\alpha\rightarrow1}g^{\textsc{Tsallis}}_{\alpha,1}(P)$ correspond to Gini impurity, the \cite{kearns1996boosting} objective and the entropy criterion respectively.
\end{proposition}

\begin{proof}
     Setting $\alpha=2,\beta=1$ immediately yields the expression for Gini impurity. Plugging $\alpha=\frac{1}{2},\beta=2$ yields

\begin{align*}
    g^{\textsc{Tsallis}}_{\frac{1}{2},2}(P)&=\frac{C}{-\frac{1}{2}}\left(1-\left(\sum_{i=1}^c\sqrt{p_i}\right)^2\right)\\
    &=2C\left(\sum_{i=1}^cp_i+2\sum_{i\ne j}\sqrt{p_ip_j}-1\right)\\
    &=4C\sum_{i\ne j}\sqrt{p_ip_j}.
\end{align*}

For $c=2$, $g^{\textsc{Tsallis}}_{\frac{1}{2},2}(P)=4C\sqrt{p_1(1-p_1)}$ which matches the splitting function of \cite{kearns1996boosting}. Also taking the limit $\alpha\rightarrow1$ gives

\begin{align*}
    g^{\textsc{Tsallis}}_{\alpha\rightarrow 1,\beta}(P)&=\lim_{\alpha\rightarrow 1}\frac{C}{\alpha-1}\left(1-\left(\sum_{i=1}^c{p_i}^\alpha\right)^\beta\right)\\
    &=-C\beta\left(\sum_{i=1}^cp_i^\alpha\right)^{\beta-1}\left(\sum_{i=1}^cp_i^\alpha\ln p_i\right)\\
    &=-C\beta\left(\sum_{i=1}^cp_i\ln p_i\right).
\end{align*}

\noindent For $\beta=1$, this corresponds to the entropy criterion.
\end{proof}

\noindent We further show that the $g^{\textsc{Tsallis}}_{\alpha,\beta}(P)$ family of splitting criteria enjoys the property of being {\it permissible} splitting criteria (in the sense of \cite{kearns1996boosting}) for any $\alpha\in\R^+,\beta\in\mathbb{Z}^+,\alpha\notin (1/\beta,1)$, which implies useful desirable guarantees (e.g.\ ensuring convergences of top-down learning) for the top-down decision tree learner \cite{kearns1996boosting,de2015splitting}. 

\begin{proposition}\label{prop:permissible}
    $(\alpha,\beta)$-Tsallis entropy has the following properties for any $\alpha\in\R^+,\beta\in\mathbb{Z}^+,\alpha\notin (1/\beta,1)$
    \begin{enumerate}
        \item (Symmetry) For any $P=\{p_i\}$, $Q=\{p_{\pi(i)}$ for some permutation $\pi$ over $[c]\}$, $g^{\textsc{Tsallis}}_{\alpha,\beta}(Q)=g^{\textsc{Tsallis}}_{\alpha,\beta}(P)$.
        \item $g^{\textsc{Tsallis}}_{\alpha,\beta}(P)=0$ at any vertex $p_i=1,p_j=0$ for all $j\ne i$ of the probability simplex $P$.
        \item (Concavity) $g^{\textsc{Tsallis}}_{\alpha,\beta}(aP+(1-a)Q)\ge ag^{\textsc{Tsallis}}_{\alpha,\beta}(P)+(1-a)g^{\textsc{Tsallis}}_{\alpha,\beta}(Q)$ for any $a\in[0,1]$.
    \end{enumerate}
\end{proposition}

\noindent A proof is located in Appendix \ref{app:split}. The above properties ensure that $(\alpha,\beta)$-Tsallis entropy is a permissible splitting criterion whenever $\alpha\notin(1/\beta,1)$. {This property makes the $(\alpha,\beta)$-Tsallis entropy an interesting parametric family to study and select the best splitting criterion form, but is not needed for establishing our sample complexity results.}

\subsubsection{Sample complexity of tuning $\alpha,\beta$}

We consider $\alpha\in\R^+$ and $\beta\in[B]$ for some positive integer $B$, and observe that several previously studied splitting criteria can be readily obtained by setting appropriate values of parameters $\alpha,\beta$.
We consider the problem of tuning the parameters $\alpha,\beta$ simultaneously when designing the splitting criterion, given access to multiple problem instances (datasets) drawn from some distribution $\cD$. The goal is to find parameters $\hat{\alpha},\hat{\beta}$ based on the training samples, so that on a random $D\sim\cD$, the expected loss

$$\bbE_{D\sim\cD}L(T_{\cF,(\hat{\alpha},\hat{\beta}),t},D)$$

\noindent is minimized.
We will bound the sample complexity of the  Empirical Risk Minimization (ERM) principle, which given $N$ problem samples $D_1,\dots,D_N$ computes parameters $\hat{\alpha},\hat{\beta}$ such that

\begin{equation}\label{eqn:erm-tsallis}
\hat{\alpha},\hat{\beta}=\mathrm{argmin}_{\alpha>0,\beta\in[B]}\sum_{i=1}^NL(T_{\cF,({\alpha},{\beta}),t},D_i).
\end{equation}

\noindent We obtain the following guarantee on the sample complexity of learning a near-optimal splitting criterion. {The overall argument involves an induction on the size $t$ of the tree, which has appeared in several prior works \cite{megiddo1978combinatorial,balcan2018learning,balcan2021sample,balcan2022improved}, coupled with a counting argument for upper bounding the number of parameter sub-intervals corresponding to different behaviors of Algorithm \ref{alg:td-learning} given a parameter interval corresponding to a fixed partial tree corresponding to an intermediate stage of the algorithm. }

\begin{theorem}\label{thm:pdim-alpha-beta}
Consider the $(\alpha,\beta)$-Tsallis entropy family of splitting criteria. Suppose $\alpha>0$ and $\beta\in[B]$. For any $\epsilon,\delta>0$ and any distribution $\cD$ over problem instances with $n$ examples, $O(\frac{1}{\epsilon^2}(t(\log |\cF|+\log t+c\log(B+c))+\log\frac{1}{\delta}))$ samples  drawn from $\cD$ are sufficient to ensure that with probability at least $1-\delta$ over the draw of the samples, the parameters $\hat{\alpha},\hat{\beta}$ learned by ERM over the sample  have expected loss  that is at most $\epsilon$ larger than the expected loss of the best parameters $\alpha^*,\beta^*=\mathrm{argmin}_{\alpha>0,\beta\ge 1}\bbE_{D\sim\cD}[L(T_{\cF,(\hat{\alpha},\hat{\beta}),t},D)]$  over $\cD$. Here $t$ is the size of the decision tree, $\cF$ is the node function class used to label the nodes of the decision tree and $c$ is the number of label classes.
\end{theorem}
\begin{proof}
 Since the loss is completely determined by the final decision tree $T_{\cF,({\alpha},{\beta}),t}$, it suffices to bound the number of different algorithm behaviors as one varies the hyperparameters ${\alpha},{\beta}$ in Algorithm \ref{alg:td-learning}. As the tree is grown according to the top-down algorithm, suppose the number of internal nodes is $\tau<t$. There are $\tau+1$ candidate leaf nodes to split and $|\cF|$ candidate node functions, for a total of $(\tau+1)|\cF|$ choices for $(l,f)$. For any of ${(\tau+1)|\cF|\choose 2}$ pair of candidates $(l_1,f_1)$ and $(l_2,f_2)$, the preference for which candidate is `best' and selected for splitting next is governed by the splitting functions $G_{\alpha,\beta}(T_{l_1\rightarrow f_1})$ and $G_{\alpha,\beta}(T_{l_2\rightarrow f_2})$. 
    This preference flips across boundary condition given by $\sum_{l\in \mathrm{leaves}(T_{l_1\rightarrow f_1})}w(l)g_{\alpha,\beta}(\{p_i(l)\})=\sum_{l\in \mathrm{leaves}(T_{l_2\rightarrow f_2})}w(l)g_{\alpha,\beta}(\{p_i(l)\})$. Most terms (all but three) cancel out on both sides as we substitute a single leaf node by an internal node on both LHS and RHS. The only unbalanced terms correspond to deleted leaves $l_1,l_2$ and newly introduced leaves $l_1^a,l_1^b,l_2^a,l_2^b$, i.e.
    \begin{align*}
    \sum_{l\in \{l_1^a,l_1^b,l_2\}}\!\!\!w(l)g_{\alpha,\beta}(\{p_i(l)\})=\!\!\!\sum_{l\in \{l_2^a,l_2^b,l_1\}}\!\!\!w(l)g_{\alpha,\beta}(\{p_i(l)\}),
    \end{align*}
    where $g_{\alpha,\beta}(\cdot)=g^{\textsc{Tsallis}}_{\alpha,\beta}(\cdot)$, the $(\alpha,\beta)$-Tsallis entropy. {Note that here $w(l)$ is a constant for a fixed problem instance (independent of the parameters $\alpha,\beta$ given the structure of the tree).} For integer $\beta$, by the multinomial theorem, $(\sum_{i=1}^cp_i(l)^\alpha)^\beta$ consists of at most ${\beta+c-1\choose c}$ distinct terms. By Rolle's theorem (more preciely, Lemma \ref{lem:exp-roots}), the number of distinct solutions of the above equation in $\alpha$ is $O((\beta+c)^{c})$. Thus, for any fixed $\beta$ and fixed partial decision tree built in $\tau$ rounds, the number of critical points of $\alpha$ at which the $\mathrm{argmax}$ in Line 3 of Algorithm \ref{alg:td-learning} changes is at most $O(|\cF|^2\tau^2(\beta+c)^{c})$ and a fixed leaf node is split and labeled by a fixed $f$ for any interval of $\alpha$ induced by these critical points. {Using a simple inductive argument over the number of rounds $t$ of Algorithm \ref{alg:td-learning},} this corresponds to at most $O(\Pi_{\tau=1}^t|\cF|^2\tau^2(\beta+c)^{c})$ critical points across which the algorithmic behaviour (sequence of choices of node splits in Algorithm \ref{alg:td-learning}) can change as $\alpha$ is varied for a fixed $\beta$. Adding up over $\beta\in[B]$, we get $O(\sum_{\beta=1}^B|\cF|^{2t}t^{2t}(\beta+c)^{ct})$, or at most $O(B|\cF|^{2t}t^{2t}(B+c)^{ct})$ critical points.

    This implies a bound of $O(t(\log |\cF|+\log t+c\log(B+c)))$ on the pseudodimension of the loss function class by using Lemma \ref{lem:ddad-1}. Finally, an application of Theorem \ref{thm:pdim} completes the proof.
    \end{proof}

\noindent Observe that parameter $\alpha$ is tuned over a continuous domain and our near-optimality guarantees hold over the entire continuous domain (as opposed to say over a finite grid of $\alpha$ values). Our results have implications for cross-validation since typical cross-validation can be modeled via a distribution $\cD$ created by sampling splits from the same fixed dataset, in which case our results imply how many splits are sufficient to converge to within $\epsilon$ error of best the parameter learned by the cross validation procedure. Similar convergence guarantees have been shown for tuning the regularization coefficients of the elastic net algorithm for linear regression via cross-validation \cite{balcan2022provably,balcan2023new}. Our setting is of course more general than just cross validation and includes the case where the different datasets come from related similar tasks for which we seek to learn a common good choice of hyperparameters.

\subsubsection{Efficient implementation in output-sensitive time} 

We have discussed so far the sample complexity of the ERM algorithm for several popularly used families of tree models. However given a sample $D_1,\dots,D_N$, how does one actually efficiently implement the ERM? In this section, we propose computationally efficient approaches for computing the best hyperparameter on a given sample, for which the above sample complexity results apply.

\paragraph{Learning to split in output-sensitive time.} Our goal is to find best hyperparameters on a data sample, i.e. $\hat{\alpha},\hat{\beta}=\mathrm{argmin}_{\alpha>0,\beta\in[B]}\sum_{i=1}^NL(T_{\cF,({\alpha},{\beta}),t},D_i).$ For the 0-1 classification loss $L$, our proof of Theorem \ref{thm:pdim-alpha-beta} establishes that $L$ is piecewise constant as a function of $\alpha$, for any fixed value of $\beta$. However the upper bounds on the number of pieces in this piecewise constant function are worst-case exponential in $t$, the size of the tree. A naive approach to implement the ERM would be to compare all pairs of node functions and candidate nodes for splitting, and compute all the potential critical points for each comparison by finding all the $O((\beta+c)^{c})$ roots of an exponential equation in $\alpha$. However, only a small number (possibly zero) of these critical points may actually be relevant depending on the structure of the current tree (which may be different for different values of $\alpha$) at any stage $\tau<t$ (the number of current internal nodes). Therefore by keeping track of where we are in the top-down tree construction, the computational cost can be reduced to scale with the actual critical points instead of the potential candidates. This leads to an {\it output-sensitive} runnting time, and the approach has previous been used to improve the running times of several problems including clustering, pricing and sequence alignment \cite{balcan2020learning,balcan2024accelerating}. The key difference is that prior approaches were only developed for piecewise-structured losses with piece boundaries given by linear equations, while here we extend the approach to boundaries involving exponential equations.

\begin{algorithm}[tb]
\caption{Output-senstive implementation of ERM in (\ref{eqn:erm-tsallis})}
\label{alg:execution-tree}
\flushleft    \textbf{Input}: Datasets $D_1,\dots,D_N$\\
\textbf{Parameters}: Node function class $\cF$, tree size $t$\\
\textbf{Output}: $\hat{\alpha},\hat{\beta}$
\begin{algorithmic}[1] 
\STATE Let $D\gets D_1$.
\STATE Let $r$ be the root node of the execution tree with $r.T = \{\}$ (current tree, initially empty) and $r.I = (0,\infty)$. 
\STATE Let $s$ be a stack of execution tree nodes, initially containing the root $r$.
\STATE Let $\mathcal{T} = \emptyset$ be the initially empty set of possible decision trees.
\STATE Let $\mathcal{I} = \emptyset$ be the initially empty set of intervals.
\WHILE{the stack s is not empty}{
\STATE Pop execution tree node $e$ off stack $s$.
\STATE If $e.T$ has a decision tree of size $t$, add $e.T$ to $\mathcal{T}$ and $e.I$ to $\mathcal{I}$.
\STATE Otherwise, compute the roots of the equation $$\sum_{l\in \mathrm{leaves}(e.T_{l_1\rightarrow f_1})}w(l)g_{\alpha,\beta}(\{p_i(l)\})=\sum_{l\in \mathrm{leaves}(e.T_{l_2\rightarrow f_2})}w(l)g_{\alpha,\beta}(\{p_i(l)\})$$
among all candidate leaf node pairs $(l_1,l_2)$ and pairs of node functions $(f_1, f_2)$ that lie within $e.I$ 
(e.g. using Newton's method). The roots are computed in a sequential order, considering only boundary conditions comparing with the current best candidate for $(l, f)$.
\FORALL{subintervals $(\alpha_i,\alpha_{i+1})$ of $e.I$ induced by these roots\label{step:loop}}{
\STATE Let $T'$ be the decision tree after performing the best split for that interval.
 \STATE Create node $e'$ with $e'.T=T'$ and $e'.I=(\alpha_i,\alpha_{i+1})$, and put it onto the stack $s$.
}
\ENDFOR
}
\ENDWHILE
\STATE Repeat the above for each value of $\beta$ and each dataset $D_i$.
\STATE Return $\hat{\alpha},\hat{\beta}$ by computing the loss between each successive $\alpha$-critical point for each $\beta$.\label{step:loss}
\end{algorithmic}
\end{algorithm}   

\begin{theorem}
Consider the $(\alpha,\beta)$-Tsallis entropy family of splitting criteria. The ERM in Equation (\ref{eqn:erm-tsallis}) can be implemented in time $O(R|\cF|^2t^2\mathcal{T}_e+RM)$ using Algorithm \ref{alg:execution-tree}, where $R$ is the actual number of critical points in $\sum_{i=1}^NL(\cdot,D_i)$, $|\cF|$ is the size of the node function class, $t$ is the size of tree, $M$ is the time to compute the loss of a fixed tree on a fixed dataset, and $\mathcal{T}_e$ is the time to compute a root of an exponential equation (e.g. using Newton's method). 
\end{theorem}
\begin{proof}
The correctness of Algorithm \ref{alg:execution-tree} largely follows by adapting the arguments in prior works \cite{balcan2020learning,balcan2024accelerating}, due to certain structural similarities in the algorithm families. We summarize here the key ideas, and how they are adapted to our setting. Note that to find the best $(\hat{\alpha},\hat{\beta})$, it is sufficient to compute the piecewise constant loss as a function of $\alpha$ for each fixed $\beta\in[B]$, and then find the best $(\hat{\alpha},\hat{\beta})$ over the $B$ piecewise constant functions. For a fixed dataset $D$ and any size $\tau<t$, the sequence of first $\tau$ splits performed by the top-down learning algorithm is a piecewise constant function of the parameter $\alpha$ (for any fixed $\beta$). This is more refined structure than shown in Theorem \ref{thm:pdim-alpha-beta}, where only the final loss after $t$ splits is shown to be piecewise constant. Initially $\tau = 0$, the partition of $\alpha$ is a single region containing all real values $(0,\infty)$. For each round $\tau > 0$, the piecewise constant partition for the first $\tau$ splits is a refinement of the partition for $\tau-1$ splits (to split $\tau$ times, we must have split $\tau-1$ in some fashion). We can represent this sequence of refinements using an {\it execution tree}, where each node in the tree is labeled by an interval $I\subset(0,\infty)$, the nodes at depth $\tau$ describe the partition of the parameter domain after $\tau$ splits, and the edges represent subset relationships. This tree represents all possible execution paths for all values of $\alpha$ when run on the instance $D$. In particular, each path from the root node to a leaf corresponds to one possible sequence of splits. To find the piecewise constant loss function for a dataset $D$, it is sufficient to enumerate the leaves of the execution tree and compute the corresponding losses. Algorithm \ref{alg:execution-tree} essentially achieves this enumeration by doing a depth-first traversal of the execution tree.

We will now bound the asymptotic running time. Step \ref{step:loop} of Algorithm \ref{alg:execution-tree} runs once for each possible decision tree, giving a total of $R$ iterations. Each iteration finds the optimal split by comparing $O(|\cF|^2t^2)$ candidates with time $\mathcal{T}_e$ each to compute a single interval in the solution. Finally, once the $R$  intervals are computed it takes time $O(M)$ per interval to perform step \ref{step:loss}.
\end{proof}

\noindent We anticipate two major future directions related to computationally efficient approaches for learning decision trees. Note that our sample complexity bounds involve uniform convergence arguments, and therefore imply small generalization error even for approximate ERMs. Therefore the design of approximate implementations of ERM which are computationally faster is an exciting direction for future theoretical and empirical research (e.g., see Sharma and Jones \cite{sharma2023efficiently} where this idea has been used in the context of graph-based semi-supervised learning). Also note that our sample and computational complexity results above hold for arbitrary distributions $\cD$. Another interesting  question is where one can show distribution-dependent computational efficiency for typical distributions. One approach to this can be to give a better expected (or high probability) bound on the output-size in our results above for some nice distributions. Finally, we leave extensions of the above approach beyond learning the splitting criterion in top-down learning to future work.

\subsection{$\gamma$-geometric family}
While $(\alpha,\beta)$-Tsallis entropy is well-motivated as a parameterized class of node splitting criteria as it includes several previously studied splitting criteria, and generalizes the Tsallis entropy which may be of independent interest in other applications, it involves simulatenous optimization of two parameters which can be computationally challenging. To this end, we define the following single parameter family which  interpolates known node splitting methods:

$$g_\gamma(\{p_i\}):=C\left(\Pi_ip_i\right)^\gamma,$$

\noindent where $\gamma\in(0,1]$ and $C$ is some constant. For binary classification, the setting $\gamma=\frac{1}{2}$ and $\gamma=1$ correspond to \cite{kearns1996boosting} and Gini impurity respectively, and $\gamma=1/c$ corresponds to the geometric mean of the probabilities $\{p_i\}$, for appropriate choice of $C$. It is straightforward to verify that $g_\gamma$ is permissible for all $\gamma\in(0,1]$, i.e.\ is symmetric, zero at simplical vertices and concave. We show the following improved sample complexity guarantee for tuning $\gamma$ (proof in Appendix \ref{app:split}). Note that this family is not a special case of $(\alpha,\beta)$-Tsallis entropy, but contains additional splitting functions which may work well on given domain-specific data. Also, since it has a single parameter,  it can be  easier to optimize efficiently in practice.

\begin{theorem}\label{thm:pdim-gamma}
Consider the $\gamma$-geometric family of splitting criteria. Suppose $\gamma\in(0,1]$. For any $\epsilon,\delta>0$ and any distribution $\cD$ over problem instances with $n$ examples, $O(\frac{1}{\epsilon^2}(t(\log |\cF|+\log t)+\log\frac{1}{\delta}))$ samples  drawn from $\cD$ are sufficient to ensure that with probability at least $1-\delta$ over the draw of the samples, the parameter $\hat{\gamma}$ learned by ERM over the sample is $\epsilon$-optimal, i.e.\ has expected loss  at most $\epsilon$ larger than that of the optimal parameter  over $\cD$.
\end{theorem}

\begin{proof}
The loss is completely determined by the final decision tree $T_{\cF,\gamma,t}$. It suffices to bound the number of different algorithm behaviors as one varies the hyperparameter $\gamma$ in Algorithm \ref{alg:td-learning}. As the tree is grown according to the top-down algorithm, suppose the number of internal nodes is $\tau<t$. There are $\tau+1$ candidate leaf nodes to split and $|\cF|$ candidate node functions, for a total of $(\tau+1)|\cF|$ choices for $(l,f)$. For any of ${(\tau+1)|\cF|\choose 2}$ pair of candidates $(l_1,f_1)$ and $(l_2,f_2)$, the preference for which candidate is `best' and selected for splitting next is governed by the splitting functions $G_{\gamma}(T_{l_1\rightarrow f_1})$ and $G_{\gamma}(T_{l_2\rightarrow f_2})$. 
    This preference flips across boundary condition given by $\sum_{l\in \mathrm{leaves}(T_{l_1\rightarrow f_1})}w(l)g_{\gamma}(\{p_i(l)\})=\sum_{l\in \mathrm{leaves}(T_{l_2\rightarrow f_2})}w(l)g_{\gamma}(\{p_i(l)\})$. Most terms (all but three) cancel out on both sides as we substitute a single leaf node by an internal node on both LHS and RHS. The only unbalanced terms correspond to deleted leaves $l_1,l_2$ and newly introduced leaves $l_1^a,l_1^b,l_2^a,l_2^b$, i.e.
    \begin{align*}
    \sum_{l\in \{l_1^a,l_1^b,l_2\}}\!\!\!w(l)g_{\gamma}(\{p_i(l)\})=\!\!\!\sum_{l\in \{l_2^a,l_2^b,l_1\}}\!\!\!w(l)g_{\gamma}(\{p_i(l)\}).
    \end{align*}
    Recall $g_\gamma(\{p_i\}):=C\left(\Pi_ip_i\right)^\gamma,$ which implies that the above equation has six (i.e.\ $O(1)$) terms. By Rolle's theorem, the number of distinct solutions of the above equation in $\gamma$ is $O(1)$. Thus, the number of critical points of $\gamma$ at which the $\mathrm{argmax}$ in Line 3 of Algorithm \ref{alg:td-learning} changes is at most $O(|\cF|^2\tau^2)$ and a fixed leaf node is split and labeled by a fixed $f$ for any interval of $\gamma$ induced by these critical points. Over $t$ rounds, this corresponds to at most $O(\Pi_{\tau=1}^t|\cF|^2\tau^2)=O(|\cF|^{2t}t^{2t})$ critical points across which the algorithmic behaviour (sequence of choices of node splits in Algorithm \ref{alg:td-learning}) can change as $\gamma$ is varied. 
This implies a bound of $O(t(\log |\cF|+\log t))$ on the pseudodimension of the loss function class using Lemma \ref{lem:ddad-1}. An application of Theorem \ref{thm:pdim} completes the proof.
\end{proof}

\noindent We remark that it is straightforward to adapt Algorithm \ref{alg:execution-tree} to achieve an output-sensitive computationally efficient implementation of the ERM in the above result.

\subsection{Bayesian decision tree models}

Several Bayesian approaches for building a decision tree have been proposed in the literature \cite{chipman1998bayesian,chipman2002bayesian,wu2007bayesian}. The key idea is to specify a prior which induces a posterior distribution and a stochastic search is performed using Metropolis-Hastings algorithms to explore the posterior and find an effective tree. We will summarize the overall approach below and consider the problem of tuning parameters in the prior, which control the accuracy and size of the tree. Unlike most of prior research on data-driven algorithm design which study deterministic algorithms, we will analyze the learnability of parameters in a randomized algorithm. One notable exception is the study of random initialization of centers in $k$-center clustering via parameterized Lloyd's families \cite{balcan2018data}.

\paragraph{$\sigma,\phi$-Bayesian  family.} Let $F=(f_1,\dots,f_t)$ denote the node functions at the nodes of the decision tree $T$. The prior $p(F,T)$ is  specified using the relationship

$$p(F,T)=p(F|T)p(T).$$

\noindent We start with a tree $T$ consisting of a single root node. For any node $\tau$ in $T$, it is split with probability $p_{\textsc{split}}(\tau)=\sigma(1+d_\tau)^{-\phi}$, and if split, the process is repeated for the left and right children. Here $d_{\tau}$ denotes the depth of node $\tau$, and $\sigma,\phi$ are hyperparameters. The size of generated tree is capped to some upper bound $t$. Intuitively, $\sigma$ controls the size of the tree and $\phi$ controls its depth. At each node, the node function is selected uniformly at random from $\cF$. This specifies the prior $p(T)$. The conjugate prior for the node functions $F=(f_1,\dots,f_t)$ is given by the standard Dirichlet distribution of dimension $c-1$ (recall $c$ is the number of label classes)  with parameter $a = (a_1,\dots, a_c), a_i > 0$. Under this prior, the label predictions are given by

$$p(y\mid X,T)=\left(\frac{\Gamma(\sum_ia_i)}{\Pi_i\Gamma(a_i)}\right)^t\prod_{j=1}^t\frac{\Pi_i\Gamma(n_{ji}+a_i)}{\Gamma(n_j+\sum_ia_i)},$$

\noindent where $n_{ji} = \sum_k \I(y_{jk} =i)$ counts the number of datapoints with label $i$ at node $j$, $n_j = \sum_in_{ji}$ and $i = 1, \dots , c$. 
$a$ is usually set as the vector $(1, \dots ,1)$ which corresponds to the uniform Dirichlet prior. Finally the stochastic search of the induced posterior is done using the  Metropolis-Hastings (MH) algorithm for simulating a Markov chain \cite{chipman1998bayesian}. {Starting from a single root node, the initial tree $T^0$ is grown according to the prior $p(T)$. Then to construct $T^{i+1}$ from $T^i$, a new tree $T^*$ is constructed by splitting a random node using a random node function, pruning a random node, reassigning a node function or swapping  the node functions of a parent and a child node. Then we set $T^{i+1}=T^*$ with probability $q(T^i,T^*)$ according to the posterior $p(y\mid X,T)$, or keep $T^{i+1}=T^i$ otherwise. The algorithm outputs the tree $T^{\omega}$ where $\omega$ is typically a fixed large number of iterations (say 10000) to ensure that the search space is explored sufficiently well.} 

\paragraph{Hyperparameter tuning.} We consider the problem of tuning of prior hyperparameters $\sigma,\phi$, to obtain the best expected performance of the algorithm. {To this end, we define $\z=(\z_1,\dots,\z_{t-1})\in[0,1]^{t-1}$ as the randomness used in generating the tree $T$ according to $p(T)$. Let $T_{\z,\sigma,\phi}$ denote the resulting initial tree. Let $\z'$ denote the remaining randomness used in the selecting the random node function and the stochastic search, resulting in the final tree $T(T_{\z,\sigma,\phi}, \z',\omega)$. Our goal is to learn the hyperparameters $\sigma,\phi$ which minimize the expected loss
$$\bbE_{\z,\z',\cD}L(T(T_{\z,\sigma,\phi}, \z',\omega),D),$$
where $\cD$ denotes the distribution according to which the data $D$ is sampled, and $L$ denotes the expected fraction of incorrect predictions by the learned Bayesian decision tree.}  {ERM over a sample $D_1,\dots,D_n\sim\cD^n$ finds the parameters $\hat{\sigma},\hat{\phi}$ which minimize the expected average loss $\frac{1}{n}\sum_{i=1}^n\bbE_{\z,\z'}L(T(T_{\z,\sigma,\phi}, \z',\omega),D_i)$ over the problem instances in the sample. It is not clear how to efficiently implement this procedure. However, we can bound its sample complexity and prove the following  guarantee for learning a near-optimal prior for the Bayesian decision tree.}

\begin{theorem}\label{thm:pdim-alpha-beta-1}
Consider the $\sigma,\phi$-Bayesian  family. Suppose $\sigma,\phi>0$. {Consider the problem of designing a Bayesian decision tree learning algorithm by selecting the parameters from the $\sigma,\phi$-Bayesian algorithm family.}  For any $\epsilon,\delta>0$ and any distribution $\cD$ over problem instances with $n$ examples, $O(\frac{1}{\epsilon^2}(\log t+\log\frac{1}{\delta}))$ samples  drawn from $\cD$ are sufficient to ensure that with probability at least $1-\delta$ over the draw of the samples, the parameters $\hat{\sigma},\hat{\phi}$ learned by ERM over the sample  have expected loss  that is at most $\epsilon$ larger than the expected loss of the best parameters. Here $t$ denotes an upper bound on the size of the decision tree. 
\end{theorem}

\begin{proof}
Fix the dataset  $D$ and fix the random coins $\mathbf{z}$ used to generate the initial tree $T_{\z,\sigma,\phi}$. 
We will use the piecewise loss structure to bound the  Rademacher complexity, which would imply uniform convergence guarantees by applying standard learning-theoretic results.

{First, we establish a piecewise structure of the dual class loss for fixed prior randomization $\z'$, $\ell_{\z}^D(\sigma,\phi)=\bbE_{\z'}[L(T(T_{\z,\sigma,\phi}, \z',\omega),D)]$. Notice that the expected value under the remaining randomization $\z'$ is fixed, once the generated tree $T_{\z,\sigma,\phi}$ is fixed. We first give a bound on the number of pieces of distinct trees generated as $\sigma,\phi$ are varied. The decision whether a node $\tau_i$ is split is governed by whether $p_{\textsc{split}}(\tau)=\sigma(1+d_{\tau_i})^{-\phi}>\z_i$.} Thus, we get at most $t-1$ 2D curves in $\sigma,\phi$ across which the splitting decision may change. The curves are clearly monotonic. We further show that any pair of curves intersect in at most one point. Indeed, if $\sigma(1+d_{\tau_i})^{-\phi}=\z_i$ and $\sigma(1+d_{\tau_j})^{-\phi}=\z_j$, then $\phi'=\log(\z_j/\z_i)/\log\left(\frac{1+d_{\tau_i}}{1+d_{\tau_j}}\right)$ and $\sigma'=\z_i(1+d_{\tau_i})^{\phi'}$ is the unique point provided $\phi'>0$. Thus the set of all curves intersects in at most ${t-1\choose 2}<t^2$ points. Since the curves are planar, the number of pieces in the dual loss function (or the number of distinct trees) is also $O(t^2)$. The above argument easily extends to a collection of $N$ problem instances, with a total of at most $O(t^2N^2)$ pieces where distinct trees are generated across the instances.

Let $\rho_1, \dots, \rho_m$ denote a collection of parameter values, with one parameter from each of the $m=O(N^2t^2)$ pieces induced by all the dual class functions $\ell^{D_i}_{\z_i}(\cdot)$ for $i\in[N]$, i.e.\ across problems in the sample $\{D_1,\dots, D_N\}$ for some fixed randomizations. Let $\cH=\{f_{\rho}:(D,\mathbf{z})\mapsto l^{D}_{\mathbf{z}}(\rho)\mid \rho\in\R^+\times\R^+\}$ be a family of functions on a given sample of instances $S=\{D_i,\mathbf{z}_i\}_{i=1}^N$. Since the function $f_\rho$ is constant on each of the $m$ pieces, we have the empirical Rademacher complexity,

\begin{align*}
    \hat{R}(\cH,S):&=\frac{1}{N}\bbE_\sigma\left[\sup_{f_\rho\in\cH}\sum_{i=1}^N\sigma_i f_\rho(D_i,\mathbf{z}_i)\right]\\
    &=\frac{1}{N}\bbE_\sigma\left[\sup_{j\in[m]}\sum_{i=1}^N\sigma_i f_{\rho_j}(D_i,\mathbf{z}_i)\right]\\
    &=\frac{1}{N}\bbE_\sigma\left[\sup_{j\in[m]}\sum_{i=1}^N\sigma_i v_{ij}\right],
\end{align*}
where $\sigma=(\sigma_1,\dots,\sigma_m)$ is a tuple of i.i.d.\ Rademacher random variables, and $v_{ij}:=f_{\rho_j}(D_i,\mathbf{z}_i)$. Note that $v^{(j)}:=(v_{1j},\dots, v_{Nj})\in[0,H]^N$, and therefore $\|v^{(j)}\|_2\le H\sqrt{N}$, for all $j\in[m]$. An application of Massart's lemma \cite{massart2000some} gives 

\begin{align*}
    \hat{R}(\cH,S)&=\frac{1}{N}\bbE_\sigma\left[\sup_{j\in[m]}\sum_{i=1}^N\sigma_i v_{ij}\right] \\&\le H\sqrt{\frac{2\log m}{N}}\\&\le H\sqrt{\frac{4\log Nt}{N}}.
\end{align*}

\noindent Standard Rademacher complexity bounds (e.g.\ see \cite{mohri2018foundations}) now imply the desired sample complexity bound.

\end{proof}

\noindent So far, we have considered learning decision tree classifiers that classify any given data point into one of finitely many label classes. In the next subsection, we consider an extension of the setting to learning over regression data, for which decision trees are again known as useful interpretable models \cite{breiman1984classification}. 

\subsection{Splitting regression trees}\label{sec:regression}

In the regression problem, we have $\cY=\R$ and the top-down learning algorithm can still be used but with continous splitting criteria.
Popular splitting criteria for regression trees include the mean squared error (MSE) and half Poisson deviance (HPD). Let $y_l$ denote the set of labels for data points classified by leaf node $l$ in tree $T$ $\overline{y_l}:=\frac{1}{|y_l|}\sum_{y\in y_l}y$ is the mean prediction for node $l$. MSE is defined as $g_{\mathrm{MSE}}(y_l):=\frac{1}{|y_l|}\sum_{y\in y_l}(y-\overline{y_l})^2$ and HPD as $g_{\mathrm{HPD}}(y_l):=\frac{1}{|y_l|}\sum_{y\in y_l}(y\log\frac{y}{\overline{y_l}}-y+\overline{y_l})$. These are interpolated by the mean Tweedie deviance \cite{zhou2022tweedie} error with power $p$ given by

$$g_p(y_l):=\frac{2}{|y_l|}\sum_{y\in y_l}\left(\frac{\max\{y,0\}^{2-p}}{(1-p)(2-p)}-\frac{y\overline{y_l}}{1-p}+\frac{\overline{y_l}^{2-p}}{2-p}\right),$$

\noindent where $p=0$ corresponds to MSE and the limit $p\rightarrow 1$ corresponds to HPD. We call this the {\it $p$-Tweedie splitting criterion}, and have the following sample complexity guarantee for tuning $p$ in the multiple instance setting.

\begin{theorem}\label{thm:pdim-p}
Consider the $p$-Tweedie splitting criterion for regression trees. Suppose $p\in[0,1]$. For any $\epsilon,\delta>0$ and any distribution $\cD$ over problem instances with $n$ examples, $O(\frac{1}{\epsilon^2}(t(\log |\cF|+n)+\log\frac{1}{\delta}))$ samples  drawn from $\cD$ are sufficient to ensure that with probability at least $1-\delta$ over the draw of the samples, the Tweedie power parameter  $\hat{p}$ learned by ERM over the sample is $\epsilon$-optimal. {$\cF$ here the {\it node function class}, assumed to be finite (Section \ref{sec:prelim})}.
\end{theorem}

\begin{proof}
The loss is completely determined by the final decision tree $T_{\cF,p,t}$. It suffices to bound the number of different algorithm behaviors as one varies the hyperparameter $p$ in Algorithm \ref{alg:td-learning}. As the tree is grown according to the top-down algorithm, suppose the number of internal nodes is $\tau<t$.  For any of ${(\tau+1)|\cF|\choose 2}$ pair of candidates $(l_1,f_1)$ and $(l_2,f_2)$, the preference for which candidate is ``best'' and selected for splitting next is governed by the splitting functions $G_{p}(T_{l_1\rightarrow f_1})$ and $G_{p}(T_{l_2\rightarrow f_2})$. 
    This preference flips across boundary condition given by $\sum_{l\in \mathrm{leaves}(T_{l_1\rightarrow f_1})}w(l)g_{p}(\{p_i(l)\})=\sum_{l\in \mathrm{leaves}(T_{l_2\rightarrow f_2})}w(l)g_{p}(\{p_i(l)\})$.  The expression simplifies and the only remaining terms correspond to deleted leaves $l_1,l_2$ and newly introduced leaves $l_1^a,l_1^b,l_2^a,l_2^b$, i.e.
    $\sum_{l\in \{l_1^a,l_1^b,l_2\}}\!\!w(l)g_{p}(\{p_i(l)\})=\!\!\sum_{l\in \{l_2^a,l_2^b,l_1\}}\!\!w(l)g_{p}(\{p_i(l)\}).$
    
    Recall $g_p(\{p_i\})$ gives an equation in $O(|y_l|)=O(n)$ terms. By Rolle's theorem, the number of distinct solutions of the above equation in $p$ is $O(n)$. Thus, the number of critical points of $p$ at which the $\mathrm{argmax}$ in Line 3 of Algorithm \ref{alg:td-learning} changes is at most $O(|\cF|^2\tau^2n)$ and a fixed leaf node is split and labeled by a fixed $f$ for any interval of $p$ induced by these critical points. Over $t$ rounds, this corresponds to at most $O(\Pi_{\tau=1}^t|\cF|^2\tau^2n)=O(|\cF|^{2t}t^{2t}n^t)$ critical points across which the algorithmic behaviour (sequence of choices of node splits in Algorithm \ref{alg:td-learning}) can change as $p$ is varied. 
This implies a bound of $O(t(\log |\cF|+\log t+n))=O((\log |\cF|+n))$ on the pseudodimension of the loss function class using Lemma \ref{lem:ddad-1}, since $t\le n$. An application of Theorem \ref{thm:pdim} completes the proof.
\end{proof}

\noindent Since $t<n$, this indicates that tuning regression parameters typically (for sufficiently small $B,c$) has a larger sample complexity upper bound.

\section{Learning to prune}\label{sec:prune}

\paragraph{Min cost-complexity pruning family.} Some leaf nodes in a decision tree learned via the top-down learning algorithm may involve nodes that overfit to a small number of data points. This overfitting problem  in decision tree learning is typically resolved  by  pruning some of the branches and reducing the tree size \cite{breiman1984classification}. The process of growing trees to size $t$ and pruning back to smaller size $t'$ tends to produce more effective decision trees than learning a tree of size $t'$ top-down. We study the mininum cost-complexity pruning algorithm here, which involves a tunable complexity parameter $\tilde{\alpha}$, and establish bounds on the sample complexity of tuning $\tilde{\alpha}$ given access to repeated problem instances from dataset distribution $\cD$.

\begin{figure}[t]
\centering
\includegraphics[width=0.5\columnwidth]{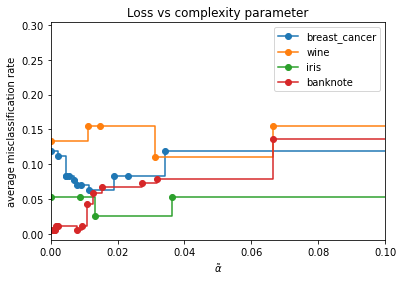}
\caption{The loss of pruned tree as a function of the mininum cost-complexity pruning parameter $\tilde{\alpha}$ is piecewise constant with at most $t$ pieces. The optimal  complexity parameter $\tilde{\alpha}$ varies with the dataset. 
}
\label{fig:pruning}
\end{figure}

The cost-complexity function for a tree $T$ is given by

$$R(T,D):=L(T,D)+\tilde{\alpha}|\mathrm{leaves}(T)|.$$

\noindent More leaf nodes correspond to higher flexibility of the decision tree in partitioning the space into smaller pieces and therefore greater ability to fit the training data. $\tilde{\alpha}\in[0,\infty)$ controls how strongly we penalize this increased complexity of the tree. The mininum cost-complexity pruning algorithm computes a subtree $T_{\tilde{\alpha}}$ of $T$ which minimizes the cost-complexity function. When $\tilde{\alpha}=0$, this selects $T$ and when $\tilde{\alpha}=\infty$ a single node tree is selected.

Given a leaf node $l$ of $T$ labeled by $i\in[c]$, the cost-complexity measure is defined to be $R(l,D)=\frac{w(l)-p_i(l)}{w(l)}+\tilde{\alpha}$. Denote by $T_t$, the branch of tree $T$ rooted at node $t$ and $R(T_t,D):=\sum_{l\in \mathrm{leaves}(T_t)}R(l,D)+\tilde{\alpha}|\mathrm{leaves}(T_t)|$. The mininum cost-complexity pruning algorithm successively deletes weakest links which minimize $\frac{R(t,D)-R(T_t,D)}{|\mathrm{leaves}(T_t)|-1}$ over internal nodes $t$ of the currently pruned tree.

We have the following result bounding the sample complexity of tuning $\tilde{\alpha}$ from multiple data samples. 

\begin{theorem}\label{thm:pdim-alpha-tilde}
Suppose $\tilde{\alpha}\in\R_{\ge0}$ and $t$ denote the size of the unpruned tree. For any $\epsilon,\delta>0$ and any distribution $\cD$ over problem instances with $n$ examples, $O(\frac{1}{\epsilon^2}(\log t+\log\frac{1}{\delta}))$ samples  drawn from $\cD$ are sufficient to ensure that with probability at least $1-\delta$ over the draw of the samples, the mininum cost-complexity pruning parameter learned by ERM over the sample is $\epsilon$-optimal.
\end{theorem}

\begin{proof}
    Fix a dataset $D$. Then there are critical values of $\tilde{\alpha}$ given by $\tilde{\alpha}_0=0<\tilde{\alpha}_1<\tilde{\alpha}_2\dots<\infty$ such that the optimal pruned tree $T_k$ is fixed for over any interval $[\tilde{\alpha}_k,\tilde{\alpha}_{k+1})$ for $k\ge 0$. Furthermore, the optimal pruned trees form a sequence of nested sub-trees $T_0=T\supset T_1\supset \dots$ (\cite{breiman1984classification}, Chapter 10). Thus, the behavior of the min cost-complexity pruning algorithm is identical over at most $t$ intervals, and the loss function is piecewise constant with at most $t$ pieces. The rest of the argument is similar to the proof of Theorem \ref{thm:pdim-alpha-beta}, and we obtain a pseudo-dimension bound of $O(\log t)$ using Lemma \ref{lem:ddad-1}. An application of Theorem \ref{thm:pdim}   implies the stated sample complexity.
\end{proof}

\noindent Minimum cost-complexity pruning \cite{breiman1984classification} can be implemented using a simple dynamic program to find the sequence of trees that minimize $R(T,D)$ for any given fixed $\tilde{\alpha}$, which takes quadratic time to implement in the size of $T$ \cite{bohanec1994trading}. Faster pruning approaches are known that directly prune nodes for which the reduction in error or splitting criterion when splitting the node is not statistically significant. This includes Critical Value Pruning \cite{mingers1987expert,mingers1989empirical} and Pessimistic Error Pruning \cite{quinlan1987simplifying}. Principled statistical learning guarantees are known for the latter \cite{mansour1997pessimistic}, and here we will consider the problem of tuning the confidence parameter in pessimistic pruning, which we describe below.

\begin{figure}[t]
\centering
\includegraphics[width=0.5\columnwidth]{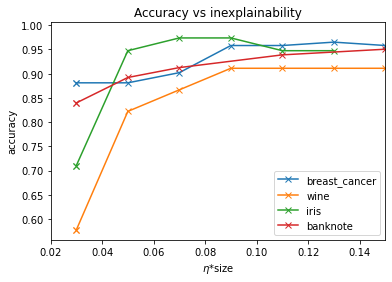} 
\caption{Accuracy vs $\eta*|\mathrm{leaves}(T)|$ as the pruning parameter $\tilde{\alpha}$ is varied, for $\eta=0.01$.}
\label{fig:exp}
\end{figure}

\paragraph{Pessimistic pruning family.} Suppose $\cX\subseteq\R^a$, i.e. each data point consists of $a$ real features or attributes. For any internal node $h$ of $T$, if $e_h$ denotes the fraction of data points that are misclassified among the $n_h$ data points that are classified via the sub-tree rooted at $h$, and $e_l$ denotes the fraction of misclassified data points if $h$ is replaced by a leaf node, then the pessimistic pruning test of  \cite{mansour1997pessimistic} is given by

$$e_l\le e_h +c_1\sqrt{\frac{t_h\log a + c_2}{n_h}},$$

\noindent where $c_1$ and $c_2$ are parameters, and $t_h$ denotes the size of the sub-tree rooted at $h$. We consider the problem of tuning $c_1,c_2$ given repeated data samples, and bound the sample complexity of tuning in the following theorem.

\begin{theorem}\label{thm:pdim-c}
Suppose $c_1,c_2\in\R_{\ge0}$ and $t$ denote the size of the unpruned tree. For any $\epsilon,\delta>0$ and any distribution $\cD$ over problem instances with $n$ examples, $O(\frac{1}{\epsilon^2}(\log t+\log\frac{1}{\delta}))$ samples  drawn from $\cD$ are sufficient to ensure that with probability at least $1-\delta$ over the draw of the samples, the pessimistic pruning parameters learned by ERM over the sample is $\epsilon$-optimal.
\end{theorem}

\begin{proof}
    For a fixed dataset $D$, the $c_1,c_2$ parameter space can be partitioned by at most $t$ algebraic curves of degree 3 that determine the result of the pessimistic pruning test. We use a general result on the pseudodimension bound in data-driven algorithm design due to \cite{bartlett2022generalization} when the loss can be computed by evaluating rational expressions to obtain a $O(\log t)$ on the pseudodimension. The result is stated below for convenience.

In this theorem, our above arguments show that there is a GJ algorithm, i.e. an algorithm which only computes and compares rational (ratios of polynomials) functions of its inputs, for computing the loss function. Here the number of real parameters $n=2$, the maximum degree of any computed expression is $\Delta=3$ and the total number of distinct predicates that need to be evaluated to compute the loss for any value of the parameters is $\Lambda=t$. Plugging into Theorem \ref{thm:gj} yields a bound of $O(\log t)$ on the pseudo-dimension, and the result follows from Theorem \ref{thm:pdim}.
\end{proof}

\noindent We have studied parameter tuning in two distinct parameterized approaches for decision tree pruning. However, several other pruning methods are known in the literature \cite{esposito1997comparative,esposito1999effects}, and it is an interesting direction for future research to design approaches to select the best method based on data. We conclude this section with a remark about another interesting future direction, namely extending our results to tree ensembles.

\section{Extension to random forests}

A popular approach to reduce overfitting of decision trees to the training set is to use {\it random forests} \cite{breiman2001random}. This approach is the instantiation of a general technique called {\it bagging} \cite{breiman1996bagging} or bootstrap aggregating to decision tree learners. Intuitively, the key idea is to train a number of decision trees on different (possibly overlapping) subsets of the training dataset, and aggregate the predictions of the different decision trees to make the prediction on any test point.

Given a dataset $D=(X,y)$, and a parameter $n_t$ denoting the number of decision trees in the random forest, one constructs $n_t$ subsets 
$D^{(1)},\dots,D^{(n_t)}$ of the dataset $D$. For each dataset $D^{(j)}$, one learns a tree $T_j:=T_{\cF,\rho,t}(D^{(j)})$, where $\rho$ denotes the hyperparameter(s) used in splitting, pruning, etc.\ for each individual tree, and $t$ denotes the size of each tree. The prediction of the random forest is given by aggregating the predictions of individual trees $T_j$, e.g.\ averaging the predictions for regression, or taking the plurality vote for classification. We denote the resulting random forest as $F_{\cF,\rho,t}(D^{(1)},\dots,D^{(n_t)})$.

Evident from the above description of random forests, one still needs to select the splitting criterion for each decision tree which impacts the performance of the learned random forest. In practice, one typically sets the same splitting criterion for all the trees in the forest \cite{pedregosa2011scikit}. We extend our result on tuning the splitting criterion by learning $\alpha,\beta$ parameters in the $(\alpha,\beta)$-Tsallis entropy (Theorem \ref{thm:pdim-alpha-beta}) to random forests below. As before, we will bound the sample complexity of the  Empirical Risk Minimization (ERM) principle, which given $N$ problem samples $D_1,\dots,D_N$ computes parameters $\hat{\alpha},\hat{\beta}$ such that
$\hat{\alpha},\hat{\beta}=\mathrm{argmin}_{\alpha>0,\beta\in[B]}\sum_{i=1}^NL(F_{\cF,({\alpha},{\beta}),t},D_i).$ Here $$L(F_{\cF,({\alpha},{\beta}),t},D_i):=\bbE_{D_i^{(1)},\dots,D_i^{(n_t)}}[L(F_{\cF,({\alpha},{\beta}),t}(D_i^{(1)},\dots,D_i^{(n_t)}),D_i)]$$ denotes the expected loss of the random forest on dataset $D_i$, where the expectation is taken over the internal randomization involved in generating the data subsets  $D_i^{(1)},\dots,D_i^{(n_t)}$.

\begin{theorem}\label{thm:pdim-alpha-beta-forest}
Suppose $\alpha>0$, $\beta\in[B]$, $n_t=o(\sqrt{n})$. For any $\epsilon,\delta>0$ and any distribution $\cD$ over problem instances with $n$ examples, $O(\frac{1}{\epsilon^2}(n_tt\log n(\log |\cF|+\log t+c\log(B+c))+\log\frac{1}{\delta}))$ samples  drawn from $\cD$ are sufficient to ensure that with probability at least $1-\delta$ over the draw of the samples, the parameters $\hat{\alpha},\hat{\beta}$ learned by ERM over the sample  have expected loss  that is at most $\epsilon$ larger than the expected loss of the best parameters $\alpha^*,\beta^*=\mathrm{argmin}_{\alpha>0,\beta\ge 1}\bbE_{D\sim\cD}L(F_{\cF,(\hat{\alpha},\hat{\beta}),t},D)$  over $\cD$ and the randomization in selecting the data subsets $D^{(1)},\dots,D^{(n_t)}$ in the forest. Here $n_t$ is the number of decision trees in the random forest, $t$ is the size of each decision tree in the forest, $\cF$ is the node function class used to label the nodes of the decision tree and $c$ is the number of label classes.
\end{theorem}
\begin{proof}
    We will describe here the main argument needed to extend the proof of Theorem \ref{thm:pdim-alpha-beta} to the above result for random forests. Note that for any fixed dataset $D^{(j)}$, the number of distinct decision trees as $\alpha,\beta$ are varied is upper bounded in the proof Theorem \ref{thm:pdim-alpha-beta} as $O(B|\cF|^{2t}t^{2t}(B+c)^{ct})$. 
    
    To determine the number of distinct random forests, it is sufficient to bound the number of ways of putting $n$ distinct balls corresponding to the datapoints into $n_t$ identical bins (since datasets  $D^{(j)}$ are treated identically when aggregating the predictions). This is given by Stirling numbers of the second kind, and when $n_t=o(\sqrt{n})$, they are asymptotically ($n\rightarrow\infty$) equal to $\frac{n^{2n_t}}{2^{n_t}n_t!}$. Combining with above, the total number of random forests may be upper bounded by $O(n^{2n_t}B|\cF|^{2t}t^{2t}(B+c)^{ct})$.

    Using Lemma \ref{lem:ddad-1}, this implies a bound of $O(n_tt\log n(\log |\cF|+\log t+c\log(B+c))$ on the pseudodimension of the dual loss function class, and the claimed sample complexity guarantee follows from standard learning theoretic results \cite{anthony1999neural}.
\end{proof}

\noindent While Theorem \ref{thm:pdim-alpha-beta-forest} is established for classification, a similar result can be established for tuning the regression splitting parameter in Section \ref{sec:regression} by combining the argument above with the proof of Theorem \ref{thm:pdim-p}.

\begin{remark}
    Random forests often also employ another bagging strategy known as  ``feature bagging'' \cite{ho1995random}. Here, at each candidate split in top-down learning, a random subset of features is selected and the best feature from the subset is selected based on the splitting criterion. The appropriate size of the random  subset of features is data-dependent \cite{hastie2009elements}, and selecting it using a data-driven approach is an interesting question for further research.
\end{remark}

\section{Learning gradient-boosted trees}
Gradient-boosted trees are widely known as the state-of-the-art for learning from tabular data, typically outperforming even neural networks \cite{chen2016xgboost,prokhorenkova2018catboost,mcelfresh2023neural}. Given  a labeled dataset $D=(X,y)$ over some input domain $X\in\cX^n$ and $y\in\R^n$, a tree ensemble model learns a collection of $K$ trees, and the prediction of the ensemble model is simply the sum of the predictions of individual trees in the model. That is,

$$\hat{y}_i=\phi(X_i)=\sum_{k=1}^KT_k(X_i),$$

\noindent where $T_k$ denotes the $k$-th tree in the ensemble and $\hat{y}_i$ denotes the prediction of the ensemble model on datapoint $X_i$. In the extremely popular XGBoost technique, the goal is to minimize a regularized objective

$$L(\phi,D)=\frac{1}{n}\sum_{i=1}^n \ell(\phi(X_i), y_i)+\frac{1}{2}\lambda \sum_{k=1}^K \|T_k\|^2.$$

\noindent Here $\ell$ denotes the loss function (typically squared loss in regression, assumed differentiable) and $\|T_k\|^2=\sum_{i=1}^t w_{ik}^2$, where $w_{ik}$ denotes the weight of leaf $i$ in tree $T_k$. While exact optimization of this objective is computationally challenging, gradient-boosting approaches build trees in the ensemble in a sequence of greedy steps that choose the split candidate that most improves the model according to the above regularized loss. Algorithm \ref{alg:xgboost} describes the process of splitting a node in XGBoost. Here $\hat{y}_i^{(k)}=\sum_{j=1}^kT_j(X_i)$ denotes the prediction of the trees at iteration $k$.

\begin{algorithm}[tb]
\caption{XGBoost for split finding at any given node \cite{chen2016xgboost}}
\label{alg:xgboost}
\flushleft\textbf{Input}: Dataset $D_I$ at node $I$, tree index $k$\\
\textbf{Parameters}: Node function class $\cF$, tree size $t$\\
\textbf{Output}: $f\in\cF$ for splitting $I$
\begin{algorithmic}[1] 
\FOR{each datapoint $(X_i,y_i)\in D_I$}
\STATE $g_i\gets \partial_{\hat{y}_i^{(k-1)}}\ell(y_i, \hat{y}_i^{(k-1)})$
\STATE $h_i\gets \partial^2_{\hat{y}_i^{(k-1)}}\ell(y_i, \hat{y}_i^{(k-1)})$
\ENDFOR
\FOR{each $f$ in $\mathcal{F}$}
\STATE $D_L,D_R\gets $ split of dataset $D_I$ according to $f$
\STATE $G_L\gets \sum_{i\in D_L} g_i$, $H_L\gets \sum_{i\in D_L} h_i$
\STATE $G_R\gets \sum_{i\in D_R} g_i$, $H_R\gets \sum_{i\in D_R} h_i$
\STATE score$(f)\gets \frac{G_L^2}{H_L+\lambda}+\frac{G_R^2}{H_R+\lambda}-\frac{G^2}{H+\lambda}$
\ENDFOR
\STATE \textbf{return} split that maximizes score$(f)$
\end{algorithmic}
\end{algorithm}

\begin{theorem}\label{thm:xgboost} Consider the problem of tuning the regularization strength $\lambda$ in the  XGBoost algorithm (Algorithm \ref{alg:xgboost}). For any distribution $\mathcal{D}$ over problem instances of size $n$, given a sample of
    $O(\frac{1}{\epsilon^2}(tK\log (tK|\mathcal{F}|)+\log\frac{1}{\delta}))$ instances drawn from $\cD$,  ERM over the drawn sample learns paramater $\hat{\lambda}$ with expected loss at most $\epsilon$ larger than the best parameter $\lambda^*$ for $\cD$. Here $t$ denotes the maximum size of any tree in the ensemble, $|\cF|$ denotes the size of the node function class, and $K$ denotes the number of trees in the ensemble.
\end{theorem}

\begin{proof}
    Our overall technique is to design a GJ algorithm \cite{goldberg1993bounding} to compute the loss of the XGBoost algorithm on any fixed dataset $D$, as a function of the regularization strength $\lambda$. We bound the complexity of the GJ algorithm (including its degree and predicate complexity), and use results due to \cite{bartlett2022generalization} to establish our sample complexity guarantee (Theorem \ref{thm:gj}).

    Let $D_L$ and $D_R$ denote the left and right splits of the dataset at some fixed node $I$ of  tree $T_i$ when using a fixed node function $f\in\mathcal{F}$. The score for this split is given by 

    $$S(I, i, f) = \frac{G_L^2}{H_L+\lambda}+\frac{G_R^2}{H_R+\lambda}-\frac{G^2}{H+\lambda},$$

    \noindent where $G_L,G_R,G$ and $H_L,H_R,H$ correspond to first and second order gradient statistics on the loss function respectively, on the left split, right split and the complete dataset. For any split decision, there are at most $t$ choices for $I$, $K$ choices for $i$ and $|\mathcal{F}|$ choices for $f$. Thus, we can define a GJ algorithm with at most $t^2K^2|\mathcal{F}|^2$ predicates to compute the best split at any stage of the top-down algorithm. Overall, we have no more than $tK$ splits. Thus, the predicate complexity of the GJ algorithm is at most $(t^2K^2|\mathcal{F}|^2)^{tK}$. The degree of the GJ algorithm is at most 6. Thus, using Theorem \ref{thm:gj}, the pseudo-dimension of the dual loss function class is $O(tK\log tK|\mathcal{F}|)$, which implies the above sample complexity bounds using classical results from learning theory \cite{anthony1999neural}.
\end{proof}

\section{Optimizing the explainability versus accuracy trade-off}\label{sec:exp}

Decision trees are often regarded as one of the preferred models when the model predictions need to be explainable. Complex or large decision trees can however not only overfit the data but also hamper model interpretability. So far we have considered parameter tuning when building or pruning the decision tree with the goal of optimizing accuracy on unseen ``test'' datasets on which the decision tree is built using the learned hyperparameters. We will consider a modified objective here which incorporates model complexity in the test objective.
That is, we seek to find hyperparameters ${\alpha},{\beta},\tilde{\alpha}$ based on the training samples, so that on a random $D\sim\cD$, the expected loss

$$L_\eta:=\bbE_{D\sim\cD}L(T,D)+\eta |\mathrm{leaves}(T)|$$

\noindent is minimized, where $\eta\ge 0$ is the complexity coefficient. This objective has been  studied in a recent line of work which designs techniques for provably optimal decision trees with high interpretability \cite{hu2019optimal,lin2020generalized}. Note that, while the objective is similar to min cost-complexity pruning, there the regularization term $\tilde{\alpha}|\mathrm{leaves}(T)|$ is added to the training objective to get the best generalization accuracy on test data. In contrast, we add the regularization term to the test objective itself and $\eta$ here is a fixed parameter that governs the balance between accuracy and explainability that the learner aims to strike.

Our approach here is to combine tunable splitting and pruning to optimize the accuracy-explainability trade-off. We set $(\alpha,\beta)$-Tsallis entropy as the splitting criterion and min cost-complexity pruning with parameter $\tilde{\alpha}$ as the pruning algorithm. We show the following upper bound on the sample complexity when simultaneously learning to split and prune.\looseness-1

\begin{table*}[t]

\caption{A comparison of the performance of different splitting criteria. The first column indicates the best $(\alpha,\beta)$ parameters for each dataset over the grid considered in Figure \ref{fig:tsallis}. Acc denotes test accuracy along with a 95\% confidence interval.}
\centering

\begin{tabular}{lccccc}
\toprule
Dataset &
Best $(\alpha^*,\beta^*)$ &
Acc$(\alpha^*,\beta^*)$ &
Acc(Gini)&
Acc(Entropy) &
Acc(KM96)\\
\midrule
Iris &
(0.5,1) &
$96.00\pm1.85$ &
 $92.99\pm1.53$ &
$93.33\pm1.07$ &
$94.67\pm 2.70$ \\
Banknote &
(2.45,2) &
$98.32\pm0.52$ &
 $97.01\pm0.59$ &
$97.30\pm1.62$ &
$97.00\pm1.79$ \\
Breast cancer &
 $(0.5,3)$&
 $94.69\pm0.77$&
$92.92\pm1.29$  &
 $93.01\pm1.05$ &
$93.27\pm1.16$ \\
Wine &
(2.15,6) &
$96.57\pm 1.88$ &
 $89.14\pm 3.18$ &
$92.57\pm 2.38$ &
$93.71\pm 2.26$ \\
\bottomrule
\end{tabular}
\label{tab:ab}
\end{table*}

\begin{theorem}\label{thm:pdim-split-prune}
Suppose $\alpha>0,\beta\in[B],\tilde{\alpha}\ge0$. For any $\epsilon,\delta>0$ and any distribution $\cD$ over problem instances with $n$ examples, $O(\frac{1}{\epsilon^2}(t(\log |\cF|+\log t+c\log(B+c))+\log\frac{1}{\delta}))$ samples  drawn from $\cD$ are sufficient to ensure that with probability at least $1-\delta$ over the draw of the samples, the parameters learned by ERM for $L_\eta$ are $\epsilon$-optimal.
\end{theorem}

\begin{proof}
As argued in the proof of Theorem \ref{thm:pdim-alpha-beta}, there is a bound of $O(B|\cF|^{2t}t^{2t}(B+c)^{ct})$ on the number of distinct algorithmic behavior of the top-down learning algorithm in growing a tree of size $t$ as the parameters $\alpha,\beta$ are varied. Further, as argued in the proof of Theorem \ref{thm:pdim-alpha-tilde}, for each of these learned trees, there are at most $t$ distinct pruned trees as $\tilde{\alpha}$ is varied. Overall, this corresponds to $O(B|\cF|^{2t}t^{2t+1}(B+c)^{ct})$ distinct behaviors, which implies the claimed sample complexity bound using standard tools from learning theory and data-driven algorithm design (Lemma \ref{lem:ddad-1}, Theorem \ref{thm:pdim}).\looseness-1
\end{proof}

\section{Combined algorithm and hyperparameter selection}

There has been a recent surge of research interest in developing approaches for not just tuning hyperparameters in machine learning, but selecting a combination of an algorithm and its tuned hyperparameters from a class of candidate parameterized algorithms \cite{thornton2013auto,li2020efficient}. While we study the sample complexity of learning the best hyperparameter for several different parameterized algorithm families that learn tree-based models, it is a natural question to ask how to choose the most appropriate algorithm family for a given problem domain in the first place. The vast majority of applied literature in machine learning simply recommends the ``better'' algorithm by comparing the performance on a small finite set of hyperparameters for each algorithm family under study, and lack principled insights or provable guarantees. In contrast, our results in this section will imply a provable upper bound on the sample complexity of learning the best combination of algorithm and its hyperparameter.

We will now state and prove a general result for learning the best algorithm + hyperparameter combination, given a collection of $\kappa$ algorithm families $A_1,\dots,A_{\kappa}$ defined on a common instance space $\mathcal{X}$. Suppose each algorithm family $A_i$ is associated with a hyperparameter space $\mathcal{A}_i\subset\mathbb{R}^{d_i}$ and a corresponding loss function $\ell_{i}:\cX\times \mathcal{A}_i\rightarrow[0,H]$. Further assume finite bounds on the pseudo-dimension of the loss function classes $\cL_i:=\{l_{\alpha}:\cX\rightarrow[0,H]\mid \alpha\in \mathcal{A}_i\}$, i.e.\ $\mathrm{Pdim}(\cL_i)\le \cP_i$. Then we have the following result bounding the pseudo-dimension of $\cL_{\cup k}:=\{l_{\alpha}:\cX\rightarrow[0,H]\mid \alpha\in \cup_{i=1}^k\mathcal{A}_i\}$.

\begin{theorem}\label{thm:cash}
    Let $\cL_i:=\{l_{\alpha}:\cX\rightarrow[0,H]\mid \alpha\in \mathcal{A}_i\}$ and $\cL_{\cup k}:=\{l_{\alpha}:\cX\rightarrow[0,H]\mid \alpha\in \cup_{i=1}^k\mathcal{A}_i\}$. Then $\mathrm{Pdim}(\cL_{\cup k})\le \frac{e^2}{e^2-1}\left(\log k + \max_i\mathrm{Pdim}(\cL_i)\right)$.
\end{theorem}

\begin{proof}
    Suppose $S = \{x_1, \dots, x_m\} \subseteq \cX$ is pseudo-shattered by $\cL_{\cup k}$. Then (by Definition \ref{def:pdim}) there must exist real thresholds $r_1, \dots, r_m$ such that for each $b \in \{0, 1\}^m$ there is a function $f_b$ in $\cL_{\cup k}$ with $\mathrm{sign}(f_b(x_i) - r_i) = b_i$ for $i \in [m]$. That is, the number of behaviors $|\{(\sign(f(x_1)-r_1),\dots,\sign(f(x_m)-r_m))\mid f\in \cL_{\cup k}\}|=2^m$. Now for any $\cL_i$, by Sauer's lemma (unless $m\le\mathrm{Pdim}(\cL_i)$, which is not relevant for the claimed upper bound), we have that
    $$|\{(\sign(f(x_1)-r_1),\dots,\sign(f(x_m)-r_m))\mid f\in \cL_{i}\}|\le 
    \left(\frac{em}{\mathrm{Pdim}(\cL_i)}\right)^{\mathrm{Pdim}(\cL_i)}$$
    This implies 
    \begin{align*}
        2^m&=|\{(\sign(f(x_1)-r_1),\dots,\sign(f(x_m)-r_m))\mid f\in \cL_{\cup k}\}|\\ &\le \sum_{i=1}^k|\{(\sign(f(x_1)-r_1),\dots,\sign(f(x_m)-r_m))\mid f\in \cL_{i}\}|\\
        & \le \sum_{i=1}^k \left(\frac{em}{\mathrm{Pdim}(\cL_i)}\right)^{\mathrm{Pdim}(\cL_i)}
       \\
        &\le k\left(\frac{em}{\mathrm{Pdim}(\cL_j)}\right)^{\mathrm{Pdim}(\cL_j)},
    \end{align*}
    \noindent where $j=\argmax_{i\in[k]}\left(\frac{em}{\mathrm{Pdim}(\cL_i)}\right)^{\mathrm{Pdim}(\cL_i)}$. Therefore, 
    \begin{align*}
        m&\le \log k + \mathrm{Pdim}(\cL_j)\log\left(\frac{em}{\mathrm{Pdim}(\cL_j)}\right)\\
        &\le \log k +\mathrm{Pdim}(\cL_j)\left(\frac{em}{\mathrm{Pdim}(\cL_j)}\cdot \frac{1}{e^2}+\log e^2-1\right),
    \end{align*}
    using the inequality $\log a\le ab-\log b -1$ for all $a,b>0$ \cite{anthony1999neural}. Rearranging gives,
    \begin{align*}
        m&\le \frac{e^2}{e^2-1}\left(\log k + \mathrm{Pdim}(\cL_j)\right)
    \end{align*}
    which implies $m={O}(\log k+\max_i\mathrm{Pdim}(\cL_i))$ as desired.
\end{proof}

\noindent We remark that the above result allows for a general selection for a combination of algorithm families together with their hyperparameters. We give an example instantiation below in our context above.

\begin{example}
Suppose we have to decide whether we want to use a single decision tree and tune its $(\alpha,\beta)$-splitting criterion followed by a minimum cost-complexity pruning with parameter $\tilde{\alpha}$, or whether we should use a gradient-boosted decision tree ensemble by deploying XGBoost with regularization strength $\lambda$. Our goal is to build the most accurate tuned model by striking the right balance between accuracy and interpretability as given by the loss $L_{\eta}$ (Section \ref{sec:exp}). 
Now, Theorem \ref{thm:cash}, together with Theorems \ref{thm:xgboost} and \ref{thm:pdim-split-prune}, imply that the sample complexity of combined algorithm and hyperparameter selection here is $$O\left(\frac{1}{\epsilon^2}\left(t(\max\{K\log tK|\cF|,\log t|\cF|+c\log(B+c)\})+\log\frac{1}{\delta}\right)\right).$$
\end{example}

\section{Experiments}

We examine the significance of the novel splitting techniques and the importance of designing data-driven decision tree learning algorithms via hyperparameter tuning for various benchmark datasets. We only perform small-scale simulations that can be run on a personal computer. 
The datasets used are from the UCI repository, are publicly available and are briefly described below.

\begin{figure}[t]
\centering
\includegraphics[width=0.6\textwidth]{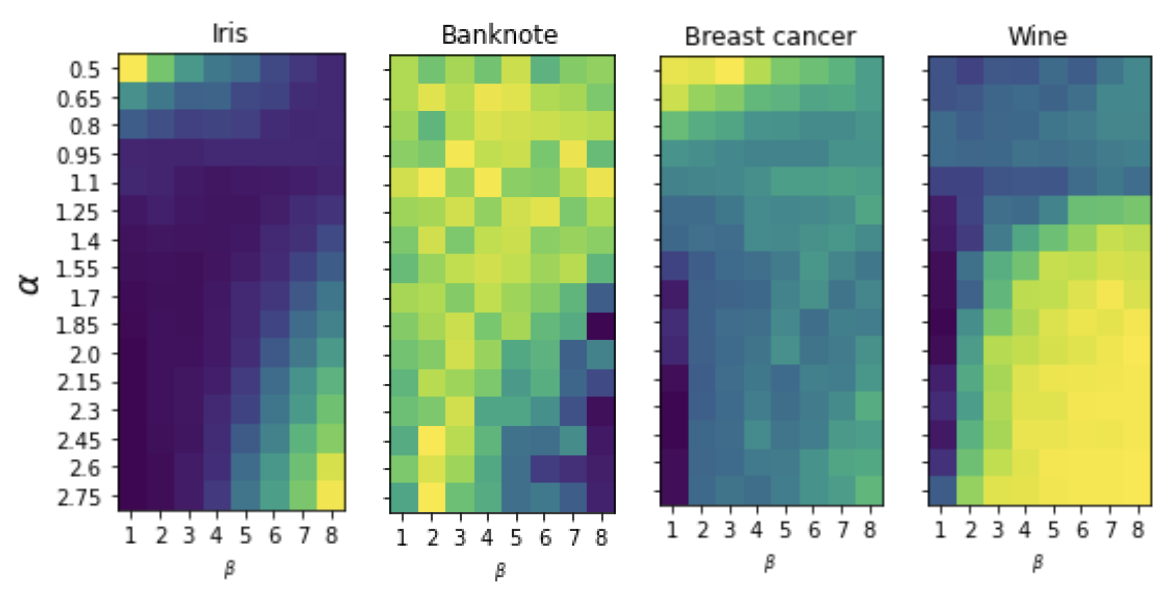} 
\caption{Average test accuracy (proportional to brightness, yellow is highest) of $(\alpha,\beta)$-Tsallis entropy  based splitting criterion as the parameters are varied, across datasets. We observe that different parameter settings work best for each dataset, highlighting the need to learn data-specific values. }
\label{fig:tsallis}
\end{figure}

\textit{Iris} \cite{fisher1936use} consists of three classes of the iris plant and four real-valued attributes. A total of 150 instances, 50 per class. \textit{Wine} \cite{lichman2013uci} has three classes of wines, 13 real attributes and 178 data points in all. \textit{Breast cancer (Wisconsin diagnostic)}  contains 569 instances, with 30 features, and two classes, malignant and benign \cite{wolberg1994machine}. The {\it Banknote Authentication} dataset \cite{misc_banknote_authentication_267} also involves binary classification and has 1372 data points and five real attributes. These datasets are selected to capture a variety of attribute sizes and number of data points.\looseness-1

\paragraph{Tuning $(\alpha,\beta)$-Tsallis entropy.} We first study the effect of choice of $(\alpha,\beta)$ parameters in the Tsallis entropy based splitting criterion. For each dataset, we perform 5-fold cross validation  for a large grid of parameters depicted in Figure \ref{fig:tsallis} and measure the accuracy on held out test set consisting of 20\% of the datapoints {(i.e.\ training datasets are just random subsets of the 80\% of the dataset used for learning the parameters)}. We implement a slightly more sophisticated variant of Algorithm \ref{alg:td-learning} which grows the tree to maximum depth of 5 (as opposed to a fixed size $t$). We do not use any pruning here. There is a remarkable difference in the optimal parameter settings for different datasets. 
Moreover, we note in Table \ref{tab:ab}, that carefully chosen values of $(\alpha,\beta)$ significantly outperform standard heuristics like Gini impurity or entropy based splitting, or even principled methods like \cite{kearns1996boosting} for which worst-case error guarantees (assuming weak learning) are known. This further underlines the significance of data-driven algorithm design for decision tree learning.\looseness-1

\begin{figure}[t]
\centering
\includegraphics[width=0.8\textwidth]{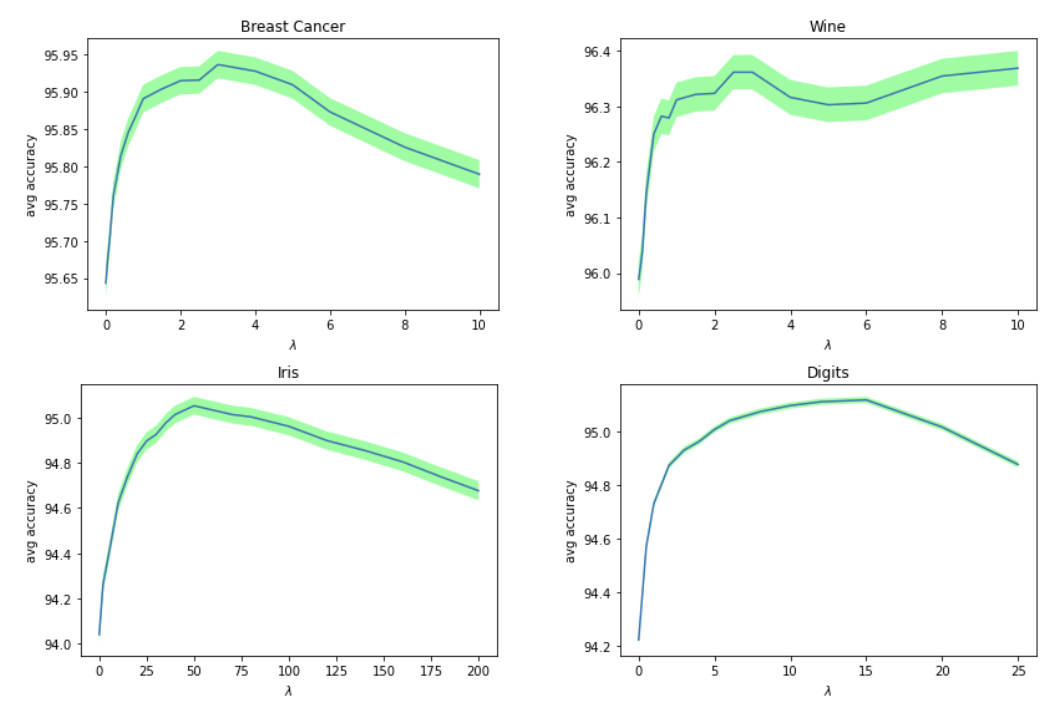} 
\caption{Average test accuracy as a function of the regularization strength $\lambda$ in XGBoost, across datasets. Different parameter settings work best for different datasets. }
\label{fig:boosting}
\end{figure}

\paragraph{Tuning the min cost-complexity parameter.} We further study the impact of tuning the complexity parameter $\tilde{\alpha}$ in the minimum cost-complexity pruning algorithm. The test error varies with $\tilde{\alpha}$ in a data dependent way and different data could have different optimal parameter as depicted in Figure \ref{fig:pruning}. We use Gini impurity as the splitting criterion. Furthermore, we observe that on a single instance, the average test error is a piecewise constant function with at most $t$ pieces which motivates the sample complexity bound in Theorem \ref{thm:pdim-alpha-tilde}.\looseness-1

\paragraph{Tuning the regularization strength $\lambda$ in XGBoost.} In Figure \ref{fig:boosting}, we show how the average test accuracy (random 80-20 split) depends on regularization parameter in XGBoost. We set $K=10$, the number of trees in the ensemble. We set the maximum depth of each tree to be 5, and set the learning rate to 1. We report the average accuracy for 10000 iterations as $\lambda$ is varied. We notice different datasets have different optimal settings for $\lambda$.\looseness-1

\paragraph{Explainability-accuracy frontier.} 
We  examine the explainability-accuracy trade-off as given by our regularized objective with complexity coefficient $\eta$. In Figure \ref{fig:exp}, we plot the explainability-accuracy frontier as the pruning parameter $\tilde{\alpha}$ is varied. Here we fix the splitting criterion as the Gini impurity. For a given dataset, this frontier can be pushed by a careful choice of the splitting criterion. 

We further study the effect of varying $\alpha$ (for fixed $\beta=1$) and $\beta$ (for fixed $\alpha=1.5$) on the explainability-accuracy trade-off. We fix $\eta=0.01$, and obtain the plot by varying the amount of pruning by changing the complexity parameter $\tilde{\alpha}$ in min-cost complexity pruning. 
We perform this study for Iris and Wine datasets in Figure \ref{fig:frontier}. We observe that for a given accuracy, the best (smallest size) explanation  could be obtained for different different splitting criteria (corresponding to settings of $\alpha,\beta$). In particular, different criteria can dominate in different regimes of size and $\eta$. Therefore, simultaneously tuning splitting criterion and pruning as in Theorem \ref{thm:pdim-split-prune} is well-motivated.

\begin{figure*}[t]
\centering
\includegraphics[width=0.77\textwidth]{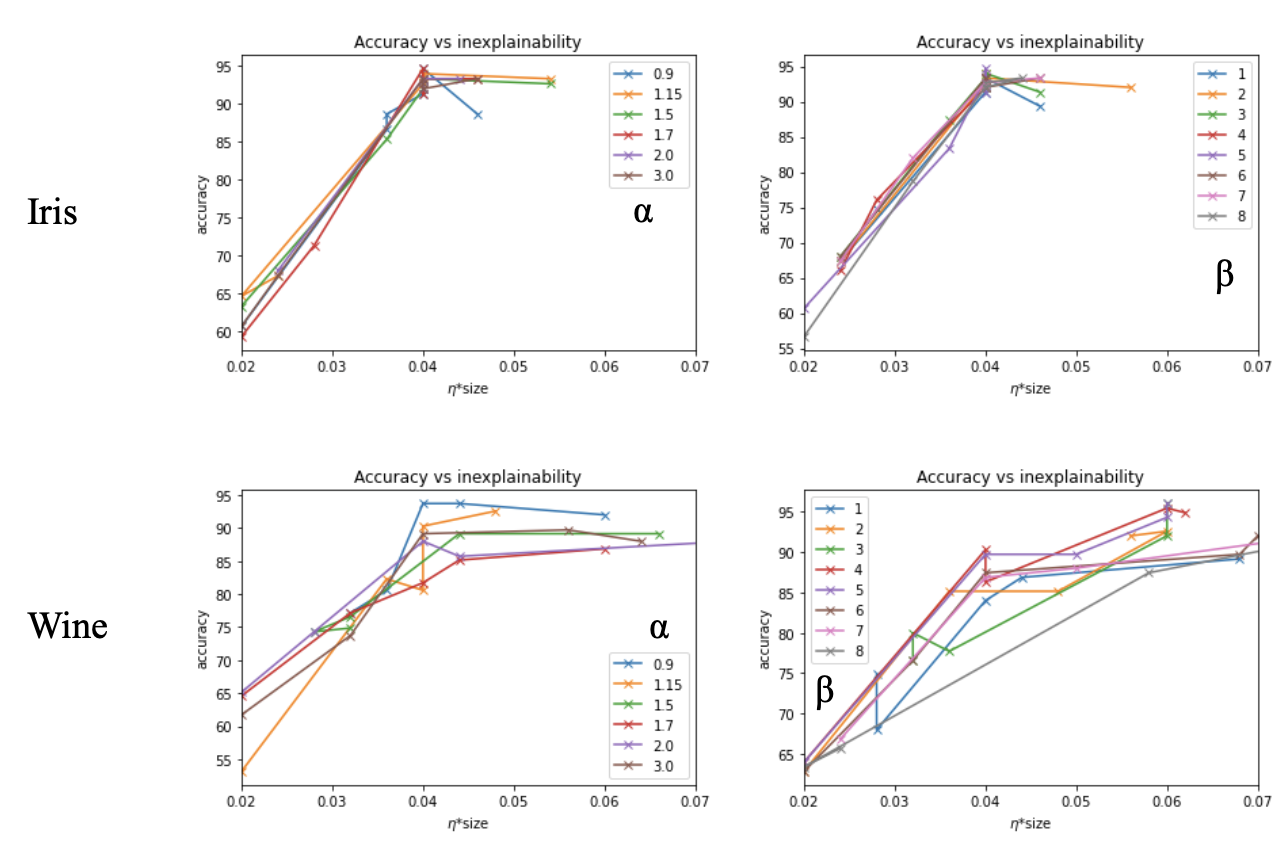} 
\caption{Accuracy-explainability frontier for different $\alpha$ or different $\beta$ in the $(\alpha,\beta)$-Tsallis entropy family for the splitting criterion, as the pruning parameter $\tilde{\alpha}$ is varied.}
\label{fig:frontier}
\end{figure*}

We defer further experiments on additional datasets and more pruning strategies to the appendix.\looseness-1

\section{Conclusion} 

We consider the problem of automatically designing decision tree based learning algorithms by data-driven selection of hyperparameters. Previous extensive research has observed that different ways to split or prune nodes when building a decision tree work best for data coming from different domain. We present a novel splitting criterion called $(\alpha,\beta)$-Tsallis entropy which interpolates popular previously known methods into a rich infinite class of algorithms. We consider the setting where we have repeated access to data from the same domain and provide formal bounds on the sample complexity of tuning the hyperparameters for the ERM principle. We extend our study to learning regression trees, selecting pruning parameters, and optimizing over the explainability-accuracy trade-off. We also establish guarantees for tuning hyperparameters in bagging and boosting based tree ensembles. We present a principled tool for analyzing the sample complexity of combined algorithm and hyperparameter selection. Empirical simulations validate our theoretical study and highlight the significance and usefulness of learning decision tree algorithms.

Our work presents several directions for future research. Our results provide extensive guarantees on sample efficiency, but it is not clear if polynomial-time computational efficiency can be achieved. Our efficient output-sensitive implementation of the ERM is a step towards understanding the computational complexity, but leaves open room for potentially more efficient implementations (or computational hardness results). 
Another direction for future research is designing and analyzing a potentially more powerful algorithm family for pruning, and extending our results on tree ensembles to other popular algorithms. We also remark that we focus on upper bounds on sample complexity, and providing corresponding lower bounds is an interesting avenue for further research.


\section*{Acknowledgements}
We thank Avrim Blum, Misha Khodak, Hedyeh Beyhaghi, Siddharth Prasad and Keegan Harris for helpful comments. This material is based on work supported by the National Science Foundation under grants CCF1910321, IIS 1901403, and SES 1919453; and the Defense Advanced Research Projects Agency under cooperative agreement HR00112020003.

\bibliographystyle{alpha}
{\bibliography{main}}

\newcommand{\etalchar}[1]{$^{#1}$}
\begin{thebibliography}{ROAG{\etalchar{+}}12}

\bibitem[AB99]{anthony1999neural}
Martin Anthony and Peter Bartlett.
\newblock {\em Neural network learning: {T}heoretical foundations}.
\newblock {Cambridge University Press}, 1999.

\bibitem[Bal20]{balcan2020data}
Maria-Florina Balcan.
\newblock {Data-Driven Algorithm Design} (book chapter).
\newblock In {\em Beyond Worst-Case Analysis of Algorithms, Tim Roughgarden (Ed)}. {Cambridge University Press}, 2020.

\bibitem[BB94]{bohanec1994trading}
Marko Bohanec and Ivan Bratko.
\newblock Trading accuracy for simplicity in decision trees.
\newblock {\em Machine Learning}, 15:223--250, 1994.

\bibitem[BB12]{bergstra2012random}
James Bergstra and Yoshua Bengio.
\newblock Random search for hyper-parameter optimization.
\newblock {\em Journal of Machine Learning Research (JMLR)}, 13(2), 2012.

\bibitem[BBSZ23]{balcan2023analysis}
Maria-Florina Balcan, Avrim Blum, Dravyansh Sharma, and Hongyang Zhang.
\newblock An analysis of robustness of non-{L}ipschitz networks.
\newblock {\em Journal of Machine Learning Research (JMLR)}, 24(98):1--43, 2023.

\bibitem[BDL20]{balcan2020learning}
Maria-Florina Balcan, Travis Dick, and Manuel Lang.
\newblock Learning to link.
\newblock In {\em International Conference on Learning Representation}, 2020.

\bibitem[BDS20]{sharma2020learning}
Maria-Florina Balcan, Travis Dick, and Dravyansh Sharma.
\newblock Learning piecewise {L}ipschitz functions in changing environments.
\newblock In {\em International Conference on Artificial Intelligence and Statistics}, pages 3567--3577. PMLR, 2020.

\bibitem[BDS21]{blum2021learning}
Avrim Blum, Chen Dan, and Saeed Seddighin.
\newblock Learning complexity of simulated annealing.
\newblock In {\em {International Conference on Artificial Intelligence and Statistics (AISTATS)}}, pages 1540--1548. PMLR, 2021.

\bibitem[BDSV18]{balcan2018learning}
Maria-Florina Balcan, Travis Dick, Tuomas Sandholm, and Ellen Vitercik.
\newblock Learning to branch.
\newblock In {\em International Conference on Machine Learning (ICML)}, pages 344--353. PMLR, 2018.

\bibitem[BDSV24]{balcan2024learning}
Maria-Florina Balcan, Travis Dick, Tuomas Sandholm, and Ellen Vitercik.
\newblock Learning to branch: Generalization guarantees and limits of data-independent discretization.
\newblock {\em Journal of the ACM}, 71(2):1--73, 2024.

\bibitem[BDV18]{balcan2018dispersion}
Maria-Florina Balcan, Travis Dick, and Ellen Vitercik.
\newblock Dispersion for data-driven algorithm design, online learning, and private optimization.
\newblock In {\em 2018 IEEE 59th Annual Symposium on Foundations of Computer Science (FOCS)}, pages 603--614. IEEE, 2018.

\bibitem[BDW18]{balcan2018data}
Maria-Florina Balcan, Travis Dick, and Colin White.
\newblock Data-driven clustering via parameterized {L}loyd's families.
\newblock {\em Advances in Neural Information Processing Systems (NeurIPS)}, 31, 2018.

\bibitem[BFOS84]{breiman1984classification}
Leo Breiman, JH~Friedman, RA~Olshen, and CJ~Stone.
\newblock {\em Classification and regression trees. {S}tatistics/probability series}.
\newblock Wadsworth Publishing Company, 1984.

\bibitem[BIW22]{bartlett2022generalization}
Peter Bartlett, Piotr Indyk, and Tal Wagner.
\newblock Generalization bounds for data-driven numerical linear algebra.
\newblock In {\em Conference on Learning Theory (COLT)}, pages 2013--2040. PMLR, 2022.

\bibitem[BKST21]{balcan2021learning}
Maria-Florina Balcan, Mikhail Khodak, Dravyansh Sharma, and Ameet Talwalkar.
\newblock Learning-to-learn non-convex piecewise-{L}ipschitz functions.
\newblock {\em Advances in Neural Information Processing Systems (NeurIPS)}, 34:15056--15069, 2021.

\bibitem[BKST22]{balcan2022provably}
Maria-Florina Balcan, Mikhail Khodak, Dravyansh Sharma, and Ameet Talwalkar.
\newblock Provably tuning the {ElasticNet} across instances.
\newblock {\em Advances in Neural Information Processing Systems (NeurIPS)}, 35:27769--27782, 2022.

\bibitem[BNS23]{balcan2023new}
Maria-Florina Balcan, Anh Nguyen, and Dravyansh Sharma.
\newblock New bounds for hyperparameter tuning of regression problems across instances.
\newblock {\em Advances in Neural Information Processing Systems (NeurIPS)}, 36, 2023.

\bibitem[BNS25a]{balcan2024algorithm}
Maria-Florina Balcan, Anh~Tuan Nguyen, and Dravyansh Sharma.
\newblock Algorithm configuration for structured pfaffian settings.
\newblock {\em Transactions on Machine Learning Research (TMLR), to appear}, 2025.

\bibitem[BNS25b]{balcan2025sample}
Maria-Florina Balcan, Anh~Tuan Nguyen, and Dravyansh Sharma.
\newblock Sample complexity of data-driven tuning of model hyperparameters in neural networks with structured parameter-dependent dual function.
\newblock {\em arXiv preprint arXiv:2501.13734}, 2025.

\bibitem[BPS24]{balcan2024subsidy}
Maria-Florina Balcan, Matteo Pozzi, and Dravyansh Sharma.
\newblock Subsidy design for better social outcomes.
\newblock {\em arXiv preprint arXiv:2409.03129}, 2024.

\bibitem[BPSV21]{balcan2021sample}
Maria-Florina Balcan, Siddharth Prasad, Tuomas Sandholm, and Ellen Vitercik.
\newblock Sample complexity of tree search configuration: Cutting planes and beyond.
\newblock {\em Advances in Neural Information Processing Systems (NeurIPS)}, 34:4015--4027, 2021.

\bibitem[BPSV22]{balcan2022improved}
Maria-Florina Balcan, Siddharth Prasad, Tuomas Sandholm, and Ellen Vitercik.
\newblock Improved sample complexity bounds for branch-and-cut.
\newblock In {\em International Conference on Principles and Practice of Constraint Programming (CP)}, 2022.

\bibitem[Bre96]{breiman1996bagging}
Leo Breiman.
\newblock Bagging predictors.
\newblock {\em Machine learning}, 24:123--140, 1996.

\bibitem[Bre01]{breiman2001random}
Leo Breiman.
\newblock Random forests.
\newblock {\em Machine learning}, 45:5--32, 2001.

\bibitem[BS21]{balcan2021data}
Maria-Florina Balcan and Dravyansh Sharma.
\newblock Data driven semi-supervised learning.
\newblock {\em Advances in Neural Information Processing Systems (NeurIPS)}, 34:14782--14794, 2021.

\bibitem[BSS24]{balcan2024accelerating}
Maria-Florina Balcan, Christopher Seiler, and Dravyansh Sharma.
\newblock Accelerating {ERM} for data-driven algorithm design using output-sensitive techniques.
\newblock {\em Advances in Neural Information Processing Systems (NeurIPS)}, 37:72648--72687, 2024.

\bibitem[CG16]{chen2016xgboost}
Tianqi Chen and Carlos Guestrin.
\newblock {XGBoost: A} scalable tree boosting system.
\newblock In {\em Proceedings of the 22nd ACM SIGKDD International Conference on Knowledge Discovery and Data Mining}, KDD '16, page 785–794, New York, NY, USA, 2016. Association for Computing Machinery.

\bibitem[CGM98]{chipman1998bayesian}
Hugh~A Chipman, Edward~I George, and Robert~E McCulloch.
\newblock Bayesian {CART} model search.
\newblock {\em Journal of the American Statistical Association (JASA)}, 93(443):935--948, 1998.

\bibitem[CGM02]{chipman2002bayesian}
Hugh~A Chipman, Edward~I George, and Robert~E McCulloch.
\newblock Bayesian treed models.
\newblock {\em Machine Learning}, 48:299--320, 2002.

\bibitem[CNG18]{chow2018path}
Yinlam Chow, Ofir Nachum, and Mohammad Ghavamzadeh.
\newblock Path consistency learning in {T}sallis entropy regularized {MDP}s.
\newblock In {\em International Conference on Machine Learning (ICML)}, pages 979--988. PMLR, 2018.

\bibitem[CNK{\etalchar{+}}12]{misc_seeds_236}
Magorzata Charytanowicz, Jerzy Niewczas, Piotr Kulczycki, Piotr Kowalski, and Szymon Lukasik.
\newblock {Seeds}.
\newblock UCI Machine Learning Repository, 2012.

\bibitem[DLH{\etalchar{+}}22]{demirovic2022murtree}
Emir Demirovi{\'c}, Anna Lukina, Emmanuel Hebrard, Jeffrey Chan, James Bailey, Christopher Leckie, Kotagiri Ramamohanarao, and Peter~J Stuckey.
\newblock Murtree: {O}ptimal decision trees via dynamic programming and search.
\newblock {\em Journal of Machine Learning Research (JMLR)}, 23(1):1169--1215, 2022.

\bibitem[DRCB15]{de2015splitting}
Rocco De~Rosa and Nicolo Cesa-Bianchi.
\newblock Splitting with confidence in decision trees with application to stream mining.
\newblock In {\em International Joint Conference on Neural Networks (IJCNN)}, pages 1--8. IEEE, 2015.

\bibitem[EMSK97]{esposito1997comparative}
Floriana Esposito, Donato Malerba, Giovanni Semeraro, and J~Kay.
\newblock A comparative analysis of methods for pruning decision trees.
\newblock {\em IEEE Transactions on Pattern Analysis and Machine Intelligence (TPAMI)}, 19(5):476--491, 1997.

\bibitem[EMST99]{esposito1999effects}
Floriana Esposito, Donato Malerba, Giovanni Semeraro, and Valentina Tamma.
\newblock The effects of pruning methods on the predictive accuracy of induced decision trees.
\newblock {\em Applied Stochastic Models in Business and Industry (ASMBI)}, 15(4):277--299, 1999.

\bibitem[Fis36]{fisher1936use}
Ronald~A Fisher.
\newblock The use of multiple measurements in taxonomic problems.
\newblock {\em Annals of Eugenics}, 7(2):179--188, 1936.

\bibitem[Ger87]{misc_glass_identification_42}
B.~German.
\newblock {Glass Identification}.
\newblock {UCI} Machine Learning Repository, 1987.

\bibitem[GJ93]{goldberg1993bounding}
Paul Goldberg and Mark Jerrum.
\newblock Bounding the {Vapnik-Chervonenkis} dimension of concept classes parameterized by real numbers.
\newblock In {\em Proceedings of the sixth annual conference on Computational learning theory}, pages 361--369, 1993.

\bibitem[GOV22]{grinsztajn2022tree}
L{\'e}o Grinsztajn, Edouard Oyallon, and Ga{\"e}l Varoquaux.
\newblock Why do tree-based models still outperform deep learning on typical tabular data?
\newblock In {\em 36th Conference on Neural Information Processing Systems (NeurIPS 2022) Track on Datasets and Benchmarks}, 2022.

\bibitem[GR16]{gupta2016pac}
Rishi Gupta and Tim Roughgarden.
\newblock A {PAC} approach to application-specific algorithm selection.
\newblock In {\em Innovations in Theoretical Computer Science (ITCS)}, pages 123--134, 2016.

\bibitem[Ho95]{ho1995random}
Tin~Kam Ho.
\newblock Random decision forests.
\newblock In {\em Proceedings of 3rd international conference on document analysis and recognition}, volume~1, pages 278--282. IEEE, 1995.

\bibitem[HRS19]{hu2019optimal}
Xiyang Hu, Cynthia Rudin, and Margo Seltzer.
\newblock Optimal sparse decision trees.
\newblock {\em Advances in Neural Information Processing Systems (NeurIPS)}, 32, 2019.

\bibitem[HTFF09]{hastie2009elements}
Trevor Hastie, Robert Tibshirani, Jerome~H Friedman, and Jerome~H Friedman.
\newblock {\em The elements of statistical learning: data mining, inference, and prediction}, volume~2.
\newblock Springer, 2009.

\bibitem[KARL18]{misc_cryotherapy_dataset__429}
Fahime Khozeimeh, Roohallah Alizadehsani, Mohamad Roshanzamir, and Pouran Layegh.
\newblock {Cryotherapy Dataset }.
\newblock UCI Machine Learning Repository, 2018.

\bibitem[KM96]{kearns1996boosting}
Michael Kearns and Yishay Mansour.
\newblock On the boosting ability of top-down decision tree learning algorithms.
\newblock In {\em Symposium on Theory of Computing (STOC)}, pages 459--468, 1996.

\bibitem[KOH{\etalchar{+}}23]{khodak2024meta}
Misha Khodak, Ilya Osadchiy, Keegan Harris, Maria-Florina Balcan, Kfir~Y Levy, Ron Meir, and Steven~Z Wu.
\newblock Meta-learning adversarial bandit algorithms.
\newblock {\em Advances in Neural Information Processing Systems (NeurIPS)}, 36, 2023.

\bibitem[KST24a]{koch2024fast}
Caleb Koch, Carmen Strassle, and Li-Yang Tan.
\newblock Fast decision tree learning solves hard coding-theoretic problems.
\newblock In {\em 2024 IEEE 65th Annual Symposium on Foundations of Computer Science (FOCS)}, pages 1893--1910. IEEE, 2024.

\bibitem[KST24b]{koch2024superconstant}
Caleb Koch, Carmen Strassle, and Li-Yang Tan.
\newblock Superconstant inapproximability of decision tree learning.
\newblock {\em 37th Annual Conference on Learning Theory (COLT)}, 196:1--32, 2024.

\bibitem[L{\etalchar{+}}13]{lichman2013uci}
Moshe Lichman et~al.
\newblock {UCI} machine learning repository, 2013.

\bibitem[LG19]{loyola2019black}
Octavio Loyola-Gonzalez.
\newblock Black-box vs. white-box: Understanding their advantages and weaknesses from a practical point of view.
\newblock {\em IEEE Access}, 7:154096--154113, 2019.

\bibitem[LJG{\etalchar{+}}20]{li2020efficient}
Yang Li, Jiawei Jiang, Jinyang Gao, Yingxia Shao, Ce~Zhang, and Bin Cui.
\newblock Efficient automatic {CASH} via rising bandits.
\newblock In {\em Proceedings of the AAAI Conference on Artificial Intelligence}, volume~34, pages 4763--4771, 2020.

\bibitem[LL14]{larose2014discovering}
Daniel~T Larose and Chantal~D Larose.
\newblock {\em Discovering knowledge in data: {A}n introduction to data mining}, volume~4.
\newblock John Wiley \& Sons, 2014.

\bibitem[Loh13]{misc_banknote_authentication_267}
Volker Lohweg.
\newblock {Banknote authentication}.
\newblock UCI Machine Learning Repository, 2013.

\bibitem[LR76]{laurent1976constructing}
Hyafil Laurent and Ronald~L Rivest.
\newblock Constructing optimal binary decision trees is {NP}-complete.
\newblock {\em Information processing letters}, 5(1):15--17, 1976.

\bibitem[LZH{\etalchar{+}}20]{lin2020generalized}
Jimmy Lin, Chudi Zhong, Diane Hu, Cynthia Rudin, and Margo Seltzer.
\newblock Generalized and scalable optimal sparse decision trees.
\newblock In {\em International Conference on Machine Learning (ICML)}, pages 6150--6160. PMLR, 2020.

\bibitem[Man97]{mansour1997pessimistic}
Yishay Mansour.
\newblock Pessimistic decision tree pruning based on tree size.
\newblock In {\em International Conference on Machine Learning (ICML)}, pages 195--201, 1997.

\bibitem[Mas00]{massart2000some}
Pascal Massart.
\newblock Some applications of concentration inequalities to statistics.
\newblock In {\em Annales de la Facult{\'e} des sciences de Toulouse: Math{\'e}matiques}, volume~9, pages 245--303, 2000.

\bibitem[Meg78]{megiddo1978combinatorial}
Nimrod Megiddo.
\newblock Combinatorial optimization with rational objective functions.
\newblock In {\em Symposium on Theory of Computing (STOC)}, pages 1--12, 1978.

\bibitem[Min87]{mingers1987expert}
John Mingers.
\newblock Expert systems—rule induction with statistical data.
\newblock {\em Journal of the Operational Research Society (JORS)}, 38:39--47, 1987.

\bibitem[Min89a]{mingers1989empirical}
John Mingers.
\newblock An empirical comparison of pruning methods for decision tree induction.
\newblock {\em Machine Learning}, 4:227--243, 1989.

\bibitem[Min89b]{mingers1989empiricalb}
John Mingers.
\newblock An empirical comparison of selection measures for decision-tree induction.
\newblock {\em Machine learning}, 3:319--342, 1989.

\bibitem[MKV{\etalchar{+}}23]{mcelfresh2023neural}
Duncan McElfresh, Sujay Khandagale, Jonathan Valverde, Vishak Prasad~C, Ganesh Ramakrishnan, Micah Goldblum, and Colin White.
\newblock When do neural nets outperform boosted trees on tabular data?
\newblock {\em Advances in Neural Information Processing Systems (NeurIPS)}, 36:76336--76369, 2023.

\bibitem[Mol19]{molnar2020interpretable}
Christoph Molnar.
\newblock {\em Interpretable {M}achine {L}earning}.
\newblock 2019.

\bibitem[MP93]{murphy1993exploring}
Patrick~M Murphy and Michael~J Pazzani.
\newblock Exploring the decision forest: An empirical investigation of {O}ccam's razor in decision tree induction.
\newblock {\em Journal of Artificial Intelligence Research}, 1:257--275, 1993.

\bibitem[MR14]{maimon2014data}
Oded~Z Maimon and Lior Rokach.
\newblock {\em Data mining with decision trees: theory and applications}, volume~81.
\newblock World scientific, 2014.

\bibitem[MRT18]{mohri2018foundations}
Mehryar Mohri, Afshin Rostamizadeh, and Ameet Talwalkar.
\newblock {\em Foundations of {M}achine {L}earning}.
\newblock MIT press, 2018.

\bibitem[Mur98]{murthy1998automatic}
Sreerama~K Murthy.
\newblock Automatic construction of decision trees from data: A multi-disciplinary survey.
\newblock {\em Data mining and knowledge discovery}, 2:345--389, 1998.

\bibitem[OM99]{opitz1999popular}
David Opitz and Richard Maclin.
\newblock Popular ensemble methods: An empirical study.
\newblock {\em Journal of Artificial Intelligence Research}, 11:169--198, 1999.

\bibitem[PGV{\etalchar{+}}18]{prokhorenkova2018catboost}
Liudmila Prokhorenkova, Gleb Gusev, Aleksandr Vorobev, Anna~Veronika Dorogush, and Andrey Gulin.
\newblock {CatBoost}: unbiased boosting with categorical features.
\newblock {\em Advances in Neural Information Processing Systems (NeurIPS)}, 31, 2018.

\bibitem[PVG{\etalchar{+}}11]{pedregosa2011scikit}
Fabian Pedregosa, Ga{\"e}l Varoquaux, Alexandre Gramfort, Vincent Michel, Bertrand Thirion, Olivier Grisel, Mathieu Blondel, Peter Prettenhofer, Ron Weiss, Vincent Dubourg, et~al.
\newblock Scikit-learn: Machine learning in {P}ython.
\newblock {\em Journal of Machine Learning Research (JMLR)}, 12:2825--2830, 2011.

\bibitem[Qui86]{quinlan1986induction}
J.~Ross Quinlan.
\newblock Induction of decision trees.
\newblock {\em Machine Learning}, 1:81--106, 1986.

\bibitem[Qui87]{quinlan1987simplifying}
J.~Ross Quinlan.
\newblock Simplifying decision trees.
\newblock {\em International Journal of Man-Machine Studies (IJMMS)}, 27(3):221--234, 1987.

\bibitem[Qui93]{quinlan2014c4}
J~Ross Quinlan.
\newblock {\em C4.5: Programs for Machine Learning}.
\newblock Morgan Kaufmann, 1993.

\bibitem[Qui96]{quinlan1996learning}
J.~Ross Quinlan.
\newblock Learning decision tree classifiers.
\newblock {\em ACM Computing Surveys (CSUR)}, 28(1):71--72, 1996.

\bibitem[Qui98]{quinlan1998miniboosting}
J~Ross Quinlan.
\newblock Miniboosting decision trees.
\newblock {\em Journal of Artificial Intelligence Research}, 10:1--15, 1998.

\bibitem[ROAG{\etalchar{+}}12]{misc_human_activity_recognition_using_smartphones_240}
Jorge Reyes-Ortiz, Davide Anguita, Alessandro Ghio, Luca Oneto, and Xavier Parra.
\newblock {Human Activity Recognition Using Smartphones}.
\newblock UCI Machine Learning Repository, 2012.

\bibitem[Rud18]{rudin2018please}
Cynthia Rudin.
\newblock Please stop explaining black box models for high stakes decisions.
\newblock {\em Stat}, 1050:26, 2018.

\bibitem[Rud19]{rudin2019stop}
Cynthia Rudin.
\newblock Stop explaining black box machine learning models for high stakes decisions and use interpretable models instead.
\newblock {\em Nature Machine Intelligence}, 1(5):206--215, 2019.

\bibitem[Sha24]{sharma2024no}
Dravyansh Sharma.
\newblock No internal regret with non-convex loss functions.
\newblock In {\em Proceedings of the AAAI Conference on Artificial Intelligence}, volume~38, pages 14919--14927, 2024.

\bibitem[SJ23]{sharma2023efficiently}
Dravyansh Sharma and Maxwell Jones.
\newblock Efficiently learning the graph for semi-supervised learning.
\newblock In {\em Uncertainty in Artificial Intelligence}, pages 1900--1910. PMLR, 2023.

\bibitem[SMTI19]{speiser2019comparison}
Jaime~Lynn Speiser, Michael~E Miller, Janet Tooze, and Edward Ip.
\newblock A comparison of random forest variable selection methods for classification prediction modeling.
\newblock {\em Expert systems with applications}, 134:93--101, 2019.

\bibitem[SS25]{sharma2025offline}
Dravyansh Sharma and Arun~Sai Suggala.
\newblock Offline-to-online hyperparameter transfer for stochastic bandits.
\newblock {\em Association for the Advancement of Artificial Intelligence (AAAI)}, 2025.

\bibitem[SZA22]{shwartz2022tabular}
Ravid Shwartz-Ziv and Amitai Armon.
\newblock Tabular data: Deep learning is not all you need.
\newblock {\em Information Fusion}, 81:84--90, 2022.

\bibitem[THHLB13]{thornton2013auto}
Chris Thornton, Frank Hutter, Holger~H Hoos, and Kevin Leyton-Brown.
\newblock Auto-{WEKA}: Combined selection and hyperparameter optimization of classification algorithms.
\newblock In {\em Proceedings of the 19th ACM SIGKDD international conference on Knowledge discovery and data mining}, pages 847--855, 2013.

\bibitem[Tsa88]{tsallis1988possible}
Constantino Tsallis.
\newblock Possible generalization of {Boltzmann-Gibbs} statistics.
\newblock {\em Journal of Statistical Physics}, 52:479--487, 1988.

\bibitem[WSM94]{wolberg1994machine}
William~H Wolberg, W~Nick Street, and Olvi~L Mangasarian.
\newblock Machine learning techniques to diagnose breast cancer from image-processed nuclear features of fine needle aspirates.
\newblock {\em Cancer Letters}, 77(2-3):163--171, 1994.

\bibitem[WSX16]{wang2016improving}
Yisen Wang, Chaobing Song, and Shu-Tao Xia.
\newblock Improving decision trees by {T}sallis entropy information metric method.
\newblock In {\em International Joint Conference on Neural Networks (IJCNN)}, pages 4729--4734. IEEE, 2016.

\bibitem[WTW07]{wu2007bayesian}
Yuhong Wu, H{\aa}kon Tjelmeland, and Mike West.
\newblock Bayesian {CART}: Prior specification and posterior simulation.
\newblock {\em Journal of Computational and Graphical Statistics (JCGS)}, 16(1):44--66, 2007.

\bibitem[ZQY22]{zhou2022tweedie}
He~Zhou, Wei Qian, and Yi~Yang.
\newblock Tweedie gradient boosting for extremely unbalanced zero-inflated data.
\newblock {\em Communications in Statistics-Simulation and Computation}, 51(9):5507--5529, 2022.

\bibitem[ZS21]{zimmert2021tsallis}
Julian Zimmert and Yevgeny Seldin.
\newblock Tsallis-inf: An optimal algorithm for stochastic and adversarial bandits.
\newblock {\em Journal of Machine Learning Research (JMLR)}, 22(1):1310--1358, 2021.

\end{thebibliography}

\newpage

\appendix

\section{Additional proofs and background results from Section \ref{sec:split}}\label{app:split}

We present here background from prior work in learning theory and additional proof details.

\subsection{GJ algorithm and pseudo-dimension }

A useful technique for bounding the pseudo-dimension in data-driven algorithm design is the GJ framework based approach proposed by \cite{bartlett2022generalization}. We include below the formal details for completeness.

\begin{definition}[\cite{goldberg1993bounding,bartlett2022generalization}]
    A GJ algorithm $\Gamma$ operates on real-valued inputs, and can perform two types of operations:
    \begin{itemize}
        \item Arithmetic operations of the form $v=v_0 \odot v_1$, where $\odot \in \{+, -, \times, \div\}$.
        \item Conditional statements of the form ``if $v \ge 0$ $\dots$ else $\dots$''.
    \end{itemize}
    In both cases, $v_0, v_1$ are either inputs or values previously computed by the algorithm (which are rational functions of the inputs). The {\it degree} of a GJ algorithm is the maximum degree of any rational function it computes of the inputs. The {\it predicate complexity} of a GJ algorithm is the number of distinct rational functions that appear in its conditional statements.
\end{definition}

\noindent The following theorem  due to \cite{bartlett2022generalization}  is useful in obtaining some of our pseudodimension bounds by showing a GJ algorithm that computes the loss for all values of the hyperparameters, on any fixed input instance.

\begin{theorem}[\cite{bartlett2022generalization}] \label{thm:gj}
    Suppose that each function $f \in \cF$ is specified by $n$ real parameters. Suppose that for every $x \in \cX$ and $r \in \R$, there is a GJ algorithm $\Gamma_{x, r}$ that given $f \in \cF$, returns ``true" if $f(x) \geq r$ and ``false" otherwise. Assume that $\Gamma_{x, r}$ has degree $\Delta$ and predicate complexity $\Lambda$. Then, $\mathrm{Pdim}(\cF) = O(n\log(\Delta\Lambda))$.
\end{theorem}

\subsection{Additional details for the $(\alpha,\beta)$-Tsallis entropy family}

\noindent The following lemma establishes some useful properties of the $(\alpha,\beta)$-Tsallis entropy family.

\noindent{\bf Proposition \ref{prop:permissible} (restated)} \textit{
    $(\alpha,\beta)$-Tsallis entropy has the following properties for any $\alpha\in\R^+,\beta\in\mathbb{Z}^+,\alpha\notin (1/\beta,1)$
    \begin{enumerate}
        \item (Symmetry) For any $P=\{p_i\}$, $Q=\{p_{\pi(i)}$ for some permutation $\pi$ over $[c]$, $g^{\textsc{Tsallis}}_{\alpha,\beta}(Q)=g^{\textsc{Tsallis}}_{\alpha,\beta}(P)$.
        \item $g^{\textsc{Tsallis}}_{\alpha,\beta}(P)=0$ at any vertex $p_i=1,p_j=0$ for all $j\ne i$ of the probability simplex $P$.
        \item (Concavity) $g^{\textsc{Tsallis}}_{\alpha,\beta}(aP+(1-a)Q)\ge ag^{\textsc{Tsallis}}_{\alpha,\beta}(P)+(1-a)g^{\textsc{Tsallis}}_{\alpha,\beta}(Q)$ for any $a\in[0,1]$.
    \end{enumerate}
}

\begin{proof}[Proof of Proposition \ref{prop:permissible}] Properties 1 and 2 are readily verified.
    We further show that $(\alpha,\beta)$-Tsallis entropy is concave for $\alpha,\beta>0,\alpha\beta\ge 1$. 
    
    First consider the case $\alpha\ge 1$. We use the fact that the univariate function $f(x)=x^\theta$ is convex for all $\theta\ge 1$. For  any $a\in[0,1]$, $P=\{p_i\}_{i=1}^c, Q=\{q_i\}_{i=1}^c$,
\begin{align*}
    g^{\textsc{Tsallis}}_{\alpha,\beta}(aP+(1-a)Q)&=
    \frac{C}{\alpha-1}\left(1-\left(\sum_{i=1}^c(a{p_i}+(1-a)q_i)^\alpha\right)^\beta\right)\\
    &\ge \frac{C}{\alpha-1}\left(1-\left(\sum_{i=1}^ca{p_i}^\alpha+(1-a)q_i^\alpha\right)^\beta\right)\\&=\frac{C}{\alpha-1}\left(1-\left(a\sum_{i=1}^c{p_i}^\alpha+(1-a)\sum_{i=1}^cq_i^\alpha\right)^\beta\right)\\&\ge\frac{C}{\alpha-1}\!\left(\!1-\!\left(\!a\left(\sum_{i=1}^c{p_i}^\alpha\!\right)^\beta+(1-a)\left(\sum_{i=1}^cq_i^\alpha\!\right)^\beta\right)\!\right)\\
    &=ag^{\textsc{Tsallis}}_{\alpha,\beta}(P)+(1-a)g^{\textsc{Tsallis}}_{\alpha,\beta}(Q).
\end{align*}
It remains to consider the case $0<\alpha\le 1/\beta$. In this case, we apply the reverse Minkowski's inequality and use that $\alpha\beta\le 1$ to establish concavity.
\begin{align*}
    g^{\textsc{Tsallis}}_{\alpha,\beta}(aP+(1-a)Q)&=
    \frac{C}{\alpha-1}\left(1-\left(\sum_{i=1}^c(a{p_i}+(1-a)q_i)^\alpha\right)^\beta\right)\\
    &\ge \frac{C}{\alpha-1}\left(1-\left(\left(\sum_{i=1}^c(a{p_i})^\alpha\right)^{\frac{1}{\alpha}}+\left(\sum_{i=1}^c((1-a)q_i)^\alpha\right)^{\frac{1}{\alpha}}\right)^{\alpha\beta}\right)\\
    &=\frac{C}{\alpha-1}\left(1-\left(a\left(\sum_{i=1}^c{p_i}^\alpha\right)^{\frac{1}{\alpha}}+(1-a)\left(\sum_{i=1}^cq_i^\alpha\right)^{\frac{1}{\alpha}}\right)^{\alpha\beta}\right)\\&\ge\frac{C}{\alpha-1}\!\left(\!1-\!\left(\!a\left(\sum_{i=1}^c{p_i}^\alpha\!\right)^{\frac{1}{\alpha}\cdot\alpha\beta}+(1-a)\left(\sum_{i=1}^cq_i^\alpha\!\right)^{\frac{1}{\alpha}\cdot\alpha\beta}\right)\!\right)\\
    &=ag^{\textsc{Tsallis}}_{\alpha,\beta}(P)+(1-a)g^{\textsc{Tsallis}}_{\alpha,\beta}(Q).
\end{align*}

In this case, we note that the first inequality above holds because $\alpha-1<0$ in this case and $f(x)=x^\theta$ is concave for all $0<\theta< 1$. 

In this case, we recall that $\sum_{i=1}^kp_i=\sum_{i=1}^kq_i=1$ and apply the H\"{o}lder's inequality to observe that

\begin{align*}
1=\left(\sum_{i=1}^k(a{p_i}+(1-a)q_i)\right)&\le \left(\sum_{i=1}^k(a{p_i}+(1-a)q_i)^\alpha\right)^\frac{1}{\alpha}\left(\sum_{i=1}^k1^{1-\frac{1}{\alpha}}\right)^{1-\frac{1}{\alpha}}\\
&=\left(\sum_{i=1}^k(a{p_i}+(1-a)q_i)^\alpha\right)^\frac{1}{\alpha}k^{\frac{\alpha-1}{\alpha}}\\
&\le \left(\sum_{i=1}^k(a{p_i}+(1-a)q_i)^\alpha\right)^\frac{1}{\alpha}.
\end{align*}
Further, by convexity of $l_p$-norm for any $p>1$, 
$$\left(\sum_{i=1}^k(a{x_i}+(1-a)y_i)^p\right)^{\frac{1}{p}}\le a\left(\sum_{i=1}^k{x_i}^p\right)^{\frac{1}{p}}+(1-a)\left(\sum_{i=1}^k{y_i}^p\right)^{\frac{1}{p}}.$$
Set $p'=1/p$,  $p_i=x_i^p$ and $q_i=y_i^p$,
$$\left(\sum_{i=1}^k(a{p_i}^{p'}+(1-a)q_i^{p'})^{1/p'}\right)^{p'}\le a\left(\sum_{i=1}^k{p_i}\right)^{p'}+(1-a)\left(\sum_{i=1}^k{q_i}\right)^{p'}.$$

\end{proof}

{
\noindent To prove Theorem  \ref{thm:pdim-alpha-beta}, we state below a simple helpful lemma, which is a simple consequence of the Rolle's Theorem.

\begin{lemma}[See e.g.\ Lemma 26 in \cite{balcan2021data}]\label{lem:exp-roots}
    The equation $\sum^n_{i=1} a_ie^{b_ix} = 0$ where $a_i, b_i \in \R$ has at most $n - 1$ distinct solutions $x \in \R$.
\end{lemma}

\section{Additional experiments}
We include further experiments below. In the following, first we 
will further examine the tuning of $(\alpha,\beta)$-Tsallis entropy on additional datasets. Next we describe and empirically compare different pruning methods beyond those discussed in the main body.

\begin{figure*}[t]
\centering
\includegraphics[width=0.9\textwidth]{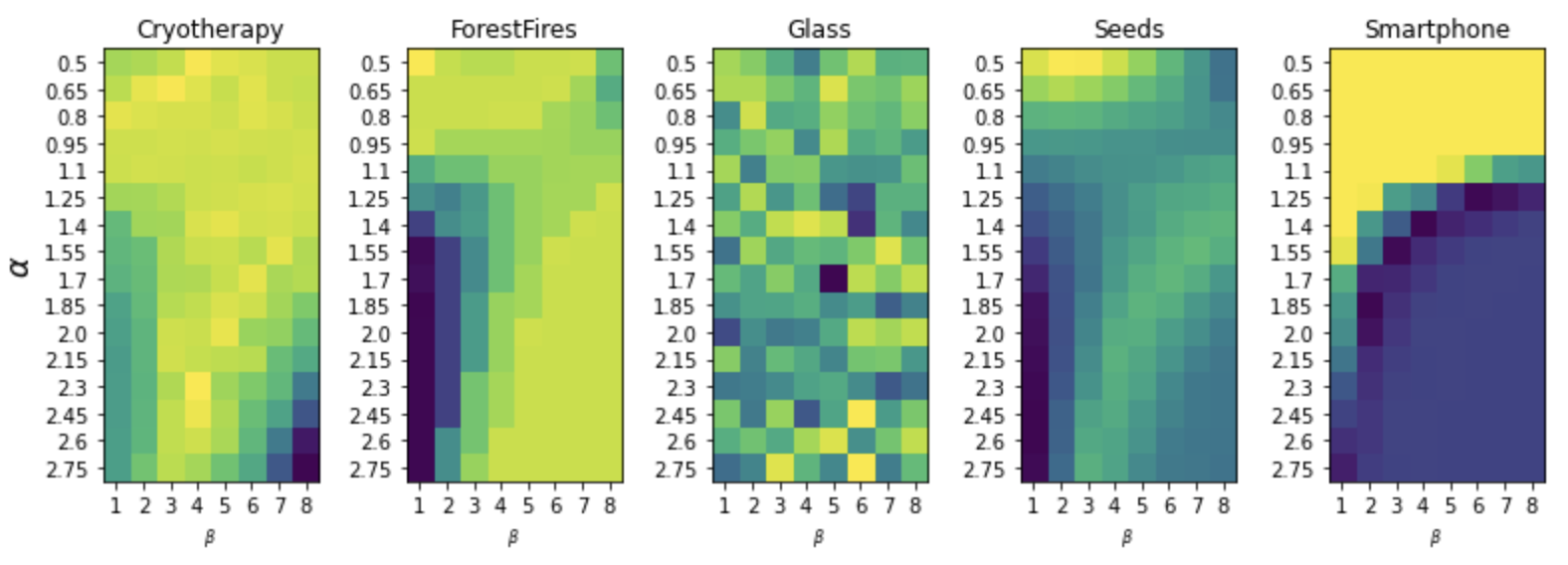} 
\caption{Average test accuracy (proportional to brightness) of $(\alpha,\beta)$-Tsallis entropy based splitting criterion across additional datasets.}
\label{fig:tsallis1}
\end{figure*}

\subsection{$(\alpha,\beta)$-Tsallis entropy on additional datasets}

We consider  additional datasets from the UCI repository and examine the best setting of $(\alpha,\beta)$ in the splitting criterion. The results are depicted in Figure \ref{fig:tsallis1} and summarized below.

Seeds \cite{misc_seeds_236} involves 3 classes of wheat, and has 210 instances with 7 attributes each. The splitting criterion due to \cite{kearns1996boosting} seems to work best here. Note that the original work only studied binary classification, and seeds involves three label classes and therefore our experiment involves a natural generalization of \cite{kearns1996boosting} to $g^{\textsc{Tsallis}}_{\frac{1}{2},2}(\cdot)$.

Cryotherapy \cite{misc_cryotherapy_dataset__429} has 90 instances with 7 real or integral attributes and contains the binary label of whether a wart was suffessfully treated using cryotherapy. Here $\alpha=0.5$ with $\beta=4$ is one of the best settings, indicating usefulness of varying the $\beta$ exponent in the KM96 criterion.

Glass identification \cite{misc_glass_identification_42} involves classification into six types of glass defined in terms of their oxide content. There are 214 instances with 9 real-valued features. Interestingly, the best performance is observed when both $\alpha$ and $\beta$ are larger than their typical values in popular criteria. For example, $(\alpha,\beta)=2.45,6$ works well here.

Algerian forest fires involves binary classification with 12 attributes and 244 instances. Gini entropy by itself does poorly, but augmented with the $\beta$-parameter the performance improves significantly and beats other candidate approaches for $\beta=8$.

Human activity detection using smartphones \cite{misc_human_activity_recognition_using_smartphones_240} is a 6-way classification dataset consisting of smartphone accelerometer and gyroscope readings corresponding to different activities, with 10299 instances with 561 features. Smaller values of $\alpha$ work better on this dataset, and the dependence on $\beta$ is weaker.

\subsection{Pruning experiments: accuracies of additional different methods}

We  examine the effectiveness of learning to prune by comparing the accuracy of pruning using the learned parameter $\tilde{\alpha}$ in  mininum cost-complexity pruning algorithm  with other baseline methods studied in the literature. Prior literature on empirical studies on pruning methods has shown that different pruning methods can work best for different datasets \cite{mingers1989empirical,esposito1997comparative}. This indicates that a practitioner should try out several pruning methods in order to obtain the best result for given domain-specific data. Here we will show that a well-tuned pruning from a single algorithm family can be competitive, and allows us to automate this process of manual selection of the pruning algorithm.

\begin{table*}[t]

\caption{A comparison of the mean test accuracy of decision trees obtained using different pruning methods.}
\centering

\begin{tabular}{l ccccc}
\toprule
Dataset &
 Acc(Unpruned) &
 Acc$(\tilde{\alpha}^*)$ in MCCP &
 Acc(REP)&
 Acc(TDP) &
 Acc(BUP)\\
\midrule
Iris &
$80.03$ &
$97.37$ &
 $96.67$ &
$90.00$ &
$93.33$ \\
Digits &
$84.44$ &
$89.42$ &
 $86.67$ &
$83.61$ &
$88.89$ \\
Breast cancer &
 $87.72$&
 $93.71$&
$92.98$  &
 $91.23$ &
$92.11$ \\
Wine &
$80.56$ &
$94.44$ &
 $91.67$ &
$88.89$ &
$86.11$ \\
\bottomrule
\end{tabular}
\label{tab:pru}
\end{table*}

We perform our experiments on benchmark datasets from the UCI repository, including Iris, Wine, Breast Cancer and Digits datasets. We split the datasets into train-test sets, using 80\% instances for training and 20\% for testing. In each case, we build the tree using entropy as the splitting criterion. We compare the mean accuracy on the test sets over 50 different splits for the following methods:

\begin{itemize}
    \item Unpruned, that is no pruning method is used.
    \item $\tilde{\alpha}^*$ in MCCP. Min-cost complexity pruning using the best parameter $\tilde{\alpha}^*$ for the dataset.
    \item REP, Reduced error pruning method of \cite{quinlan1987simplifying}.
    \item TDP, Top-down pessimistic pruning method of \cite{quinlan1986induction}.
    \item BUP, Bottom-up pessimistic pruning method of \cite{mansour1997pessimistic}.
\end{itemize}

\noindent We report our findings in Table \ref{tab:pru}. We observe that the learned pruning method has a better mean test accuracy than other baseline methods on the tested datasets.

This highlights the advantage of using a data-driven approach to tune tuning the hyperparameters---instead of trying out different pruning approaches from the literature, one could just tune the hyperparameter in the Mininum Cost-Complexity Pruning method. Further, it is an interesting question for future work to come up with more general parametric families for pruning decision trees that interpolate the different approaches from the literature that we test above.

\end{document}